\documentclass{article}

 \usepackage[preprint]{neurips_2026}

\usepackage[utf8]{inputenc} 
\usepackage[T1]{fontenc}    
\usepackage{hyperref}       
\usepackage{url}            
\usepackage{booktabs}       
\usepackage{amsfonts}       
\usepackage{nicefrac}       
\usepackage{microtype}      
\usepackage{xcolor}         

\usepackage{amsmath}
\usepackage{amssymb}
\usepackage{mathtools}
\usepackage{amsthm}
\usepackage{algorithm} 
\usepackage{enumitem}
\usepackage{graphicx}
\usepackage{subcaption}
\usepackage{array}
\usepackage{multirow}
\usepackage{booktabs}

\usepackage[capitalize,noabbrev]{cleveref}

\theoremstyle{plain}
\newtheorem{theorem}{Theorem}[section]
\newtheorem{proposition}[theorem]{Proposition}
\newtheorem{lemma}[theorem]{Lemma}

\theoremstyle{definition}
\newtheorem{definition}[theorem]{Definition}
\newtheorem{assumption}[theorem]{Assumption}
\theoremstyle{remark}

\title{
Exponential Approximation Rates and Parameter Efficiency of Learnable Bernstein Activations
}

\author{%
  Ibrahim Albool\thanks{Corresponding author.} \quad Malak Gamal El-Din \quad Salma Elmalaki \quad Yasser Shoukry \\
  Department of Electrical Engineering and Computer Science, University of California, Irvine \\
  \texttt{ialbool@uci.edu}
}

\begin{document}
\maketitle

\title{
Exponential Approximation Rates and Parameter Efficiency of Learnable Bernstein Activations
}

%

\author{%
  Ibrahim Albool\thanks{Corresponding author.} \quad Malak Gamal El-Din \quad Salma Elmalaki \quad Yasser Shoukry \\
  Department of Electrical Engineering and Computer Science, University of California, Irvine \\
  \texttt{ialbool@uci.edu}
}


 
\begin{abstract}
The choice of activation function fundamentally shapes the representational capacity and parameter efficiency of deep neural networks, yet most widely used activations lack rigorous theoretical guarantees on these properties. We provide a theoretical analysis of DeepBern-Nets (DBNs)---networks employing learnable Bernstein polynomial activations---showing that their approximation error decays with the network depth $L$ and the polynomial order $n$ with a rate of $\mathcal{O}(n^{-L})$, exponentially faster than the polynomial rate of ReLU architectures while remaining fully differentiable. We validate these predictions through $1{,}344$ experiments on large scientific datasets (HIGGS and SUSY), comparing DBNs against ReLU, Leaky ReLU, SELU, and GeLU. DBNs achieve over $70\%$ parameter reduction across the majority of architectures---reaching $99.9\%$ at scale---converge to ReLU's final loss in as few as $26\%$ of the training epochs, and attain up to $45\%$ lower final loss. These advantages hold over all tested activations, confirming that DBN's gains stem from the learnable polynomial structure rather than mere smoothness.
\end{abstract}

\section{Introduction}
\label{sec:intro}
 
The activation function is among the most consequential design choices in deep learning: it determines the class of functions a network can represent and how many parameters are needed to approximate a given target. Several works have proposed alternatives to ReLU~\cite{xu2020reluplex,hendrycks2016gaussian,li2024selu,ramachandran2017searching}, yet rigorous theoretical analysis of how the activation choice impacts representational efficiency remains scarce. In this work, we study DeepBern-Nets (DBNs)~\cite{khedr2024deepbern}, feed-forward networks where each neuron's activation is a learnable Bernstein polynomial of degree~$n$. We prove that the approximation error of a depth-$L$ DBN decays as $\mathcal{O}(n^{-L})$---exponentially faster than the polynomial rate $\mathcal{O}((W^2L^2)^{-1/d})$ of ReLU networks~\cite{yarotsky2018optimal}---while, unlike super-approximation architectures that rely on non-differentiable operators, remaining fully trainable with standard backpropagation. This stronger representation power manifests as requiring exponentially fewer parameters to achieve the same approximation quality.
 
We validate these predictions through $1{,}344$ experiments on two large datasets---\textbf{HIGGS}~\cite{baldi2014searching} (11M samples, 28 features) and \textbf{SUSY}~\cite{baldi2014searching} (5M samples, 18 features)---comparing DBNs against ReLU, Leaky ReLU, SELU, and GeLU.
Our results show that DBNs achieve over $70\%$ parameter reduction across the majority of architectures, reaching $99.9\%$ at scale---where a $2{,}226$-parameter DBN matches a $9.8$M-parameter ReLU. DBNs also exhibit faster training dynamics, reaching ReLU's final loss in as few as $24\%$ of the training epochs and attaining up to $45\%$ lower final loss, with improvements scaling with both depth and degree as predicted by our theory. An extended comparison confirms that these advantages hold over all tested smooth activations, establishing that DBN's gains stem from the learnable polynomial structure rather than mere smoothness.

\section{Preliminaries DeepBern-Nets (DBNs)}
\label{sec:prelims}


$\bullet$ \textbf{Bernstein Basis Polynomials:} A polynomial of degree $n$ can be represented in the Bernstein basis on an interval $[l, u]$ as $P_n^{[l,u]}(x) = \sum_{k=0}^n c_k b_{n,k}^{[l,u]}(x)$, where $c_k \in \mathbb{R}$ are the control coefficients. The Bernstein basis functions $b_{n,k}^{[l,u]}(x)$ are defined for $x \in [l, u]$ as:
\begin{equation}
\label{eq:bern_basis}
b_{n,k}^{[l,u]}(x) = \frac{\binom{n}{k}}{(u-l)^n}(x-l)^k (u - x)^{n - k},
\end{equation}
where $\binom{n}{k}$ is the binomial coefficient. Unlike the power basis, the Bernstein representation is intrinsic to the domain $[l, u]$, meaning the coefficients $c_k$ directly control the polynomial's geometry within this interval.

$\bullet$ \textbf{DeepBern-Nets (DBNs):}
DBNs ~\cite{khedr2024deepbern} are feed-forward neural networks where the standard activation functions are replaced by learnable Bernstein polynomials. 
For a network of depth $L$, we denote the input as $y^{(0)} = x$ and the output of the $l$-th layer as $y^{(l)}$. The propagation rule is given by:
\begin{equation}
    y^{(l)} = \sigma\left( \mathbf{W}^{(l)}y^{(l-1)} + b^{(l)}; c^{(l)} \right),
\end{equation}
where $\mathbf{W}^{(l)}$ and $b^{(l)}$ are the learnable weights and biases. The activation function $\sigma$ operates element-wise, parametrized by a set of learnable Bernstein coefficients $c^{(l)} = \{c_{k}^{(l)}\}_{k=0}^n$. Specifically, for a pre-activation scalar input $z$, the activation is defined as:
\begin{equation}
\label{eq:deepbern_activation}
\sigma(x; l, u, c^{(l)}) = \sum_{k=0}^n c_{k}^{(l)} b_{n,k}^{[l,u]}(x), \quad x \in [l, u].
\end{equation}
Here, $n$ is a hyperparameter for the polynomial degree. This formulation allows the network to learn the shape of its non-linearities alongside its weights.


$\bullet$ \textbf{Stable training of DBNs:}
Using polynomials as activation functions in deep NNs has attracted several researchers' attention in recent years~\cite{wang2022,gottemukkula2020polynomial}. A major drawback of using polynomials of arbitrary order is their unstable behavior during training due to exploding gradients--which is prominent with the increase in order~\cite{gottemukkula2020polynomial}. 
%
Luckily, and thanks to the unique properties of Bernstein polynomials, DBN does not suffer from such a limitation as captured in the proposition in the Appendix~ \ref{app:gradient}.

\section{
Superior Representation Power: DBN and the Convergence of Approximation Error
}
\label{sec:approximation_theory}

While gradient stability ensures that a network can be trained, the ultimate utility of an architecture is determined by its representation power—the efficiency with which it can approximate a target function given a fixed set of parameters. In this section, we analyze the approximation capabilities of DBNs and prove that they possess a fundamental advantage over piecewise linear models. Specifically, we establish an upper bound on the approximation error that decays exponentially with the depth of the network $L$. Unlike ReLU-based architectures, which exhibit a slower rate of error reduction relative to depth, the Bernstein basis allows for a more compact representation of complex functions. By combining these faster convergence rates with the gradient persistence established in previous sections, we demonstrate that DBNs are not only easier to optimize but are inherently more depth-efficient than their traditional counterparts.


\subsection{Network Architecture and Algebraic Structure}

To analyze the capacity of DBN, we explicitly derive the effective algebraic degree of the network. Let $\mathcal{N}: \mathbb{R}^d \to \mathbb{R}$ denote a feedforward neural network with depth $L$. We define the operation of the $l$-th layer as a composition of a linear transformation followed by a Bernstein polynomial activation.

\begin{lemma}[Effective Degree of DBNs]
\label{lemma:degree}
Consider a DBN of depth $L$, where every neuron computes a Bernstein polynomial of degree $n$. Assuming that for a given input $\mathbf{x}$, all intermediate pre-activations fall within the valid Bernstein support (e.g., via Batch Normalization), the output of the network $\mathcal{N}(x)$ is a multivariate polynomial of total degree $D$ bounded by:
\begin{equation}
    D \le n^L.
\end{equation}
\end{lemma}

The proof of Lemma~\ref{lemma:degree} is given in Appendix~\ref{app:proof_lemma_degree}.

\subsection{Approximation Bounds}

We now leverage Jackson's Inequality to bound the approximation error. To strictly isolate the benefit of depth, we analyze the capacity of the hypothesis space spanned by the network. While the exact set of polynomials representable by the network is a sub-manifold of the full polynomial space $\Pi_{n^L}$, we assume that for the target class of functions (e.g., compositional functions), the optimal approximator lies within the network's capacity.

\begin{assumption}[Sufficient Capacity]
\label{ass:capacity}
Let $f \in C(\mathbb{R}^d)$. We assume that the best polynomial approximation of $f$ of degree $n^L$, denoted $P^*_{n^L}$, lies within the hypothesis space of the network $\mathcal{N}$. Specifically, there exists a parameter configuration $\theta^*$ such that $\mathcal{N}(x; \theta^*) = P^*_{n^L}(x)$.
\end{assumption}

\begin{assumption}[Non-Choking Condition]
\label{ass:nonchoking}
To ensure the global validity of the layer-wise approximation, we assume the network width $W$ is sufficiently large to prevent information bottlenecks. Specifically, we require $W \ge d$, where $d$ is the dimension of the input manifold. As shown in~\cite{kileel2019expressive}, this guarantees that the network possesses sufficient capacity to propagate the approximated features to the next layer without rank collapse, ensuring that the total error is dominated by the polynomial degree $n$ and depth $L$.
\end{assumption}

\begin{theorem}[Exponential Approximation Rate]
\label{thm:approx_rate}
Let $f: \mathbb{R}^d \to \mathbb{R}$ be a continuous function satisfying Assumption \ref{ass:capacity}. Then, there exists a DBN $\mathcal{N}$ of depth $L$ and degree $n$ such that:
\vspace{-2mm}
\begin{equation}
    \| \mathcal{N} - f \|_\infty \le C_d \cdot \omega_f\left(\frac{1}{n^L}\right),
\end{equation}
\vspace{-1mm}
where $C_d$ is a constant that depends only on dimension $d$ and $\omega_f(x)$ is is the modulus of continuity defined in Appendix~\ref{app:mod_of_cont}.
\end{theorem}
The proof is given in Appendix ~\ref{app:thrm_approximation_proof}


\begin{table}[t]
\caption{Comparison of approximation error bounds. $\text{ReLU}_{\text{std}}$: Yarotsky \cite{yarotsky2018optimal}; $\text{ReLU}_{\text{opt}}$: Shen et al. \cite{shen2022optimal}; Floor-ReLU \cite{shen2021deep}; FLES \cite{shen2021neural}. 
}
\label{tab:approx_comparison}
\begin{center}
\begin{small}
\begin{sc}
\begin{tabular}{lccc}
\toprule
Method & Width & Depth & Error Rate \\
\midrule
$\text{ReLU}_{\text{std}}$ & $\mathcal{O}(W)$ & $\mathcal{O}(L)$ & $\omega_f\left( (W^2 L^2)^{-1/d} \right)$ \\
$\text{ReLU}_{\text{opt}}$ & $\mathcal{O}(W)$ & $\mathcal{O}(L)$ & $\omega_f\left( (W^2 L^2 \ln W)^{-1/d} \right)$ \\
$\text{ReLU}_{\text{flr}}$ & $\mathcal{O}(W)$ & $\mathcal{O}(L)$ & $\omega_f\left( W^{-\sqrt{L}} \right)$ \\
FLES & $\mathcal{O}(W)$ & $3$ & $\omega_f\left( 2^{-W} \right) + \mathcal{O}(2^{-W})$ \\
\text{DBN} & $\mathcal{O}(W)$ & $\mathcal{O}(L)$ & $\omega_f\left( n^{-L} \right)$ \\
\bottomrule
\end{tabular}
\end{sc}
\end{small}
\end{center}
\vspace*{-\baselineskip}
\end{table}

To contextualize the efficiency of DBN, we compare our derived bound against foundational results for ReLU networks and recent super-approximation architectures. Table \ref{tab:approx_comparison} summarizes these rates for a continuous function $f$ with modulus of continuity $\omega_f(\cdot)$. In the table, $\text{ReLU}_{\text{std}}$ and $\text{ReLU}_{\text{opt}}$ denote standard ReLU networks analyzed by \cite{yarotsky2018optimal} and \cite{shen2022optimal}, respectively. $\text{ReLU}_{\text{flr}}$ represents Floor-ReLU networks \cite{shen2021deep}, and FLES denotes Floor-Exponential-Step networks \cite{shen2021neural}. 

Standard ReLU networks are inherently limited by a polynomial decay in error.  \cite{yarotsky2018optimal} established that for a network of width $W$ and depth $L$, the error scales as $\omega_f((W^2 L^2)^{-1/d})$. Even with the optimized constants derived by \cite{shen2022optimal}, the fundamental polynomial rate remains unchanged. This implies that doubling the depth yields only a marginal reduction in approximation error, making high-precision approximation prohibitively expensive in terms of parameter count.

To overcome this polynomial barrier, architectures such as $\text{ReLU}_{\text{flr}}$ and FLES introduce non-differentiable operators. As shown in Table \ref{tab:approx_comparison}, utilizing the floor or step functions allows these networks to achieve root-exponential or exponential convergence rates, such as $\omega_f(2^{-W})$. However, the reliance on operators that have zero gradients almost everywhere renders these architectures unsuitable for standard gradient-based optimization, creating a gap between theoretical capacity and practical trainability.


DBN bridge this gap by achieving an exponential approximation rate of $\mathcal{O}(\omega_f(n^{-L}))$ purely through depth. Mathematically, the term $n^{-L}$ decays significantly faster than the polynomial factor $L^{-2/d}$ associated with standard ReLU networks. Crucially, unlike the super-approximation methods that rely on discontinuous functions, we achieve this rate using strictly differentiable Bernstein polynomials. This ensures that the network retains the high theoretical capacity of super-approximators while remaining amenable to efficient training via backpropagation.

\section{Experimental Results}
\label{sec:experiments}

We evaluate DBN along three axes: (i)~representational capacity on low-dimensional tasks with known ground-truth structure (\S\ref{sec:exp_capacity}), (ii)~parameter efficiency relative to ReLU baselines on the HIGGS and SUSY benchmarks (\S\ref{sec:exp_compression}), and (iii)~convergence speed during training (\S\ref{sec:exp_convergence}).
Full experimental setup, hyperparameter details, and additional results are provided in Appendix~\ref{sec:exp_details}.


\subsection{Experiment 1: Representational Capacity on Analytic Tasks}
\label{sec:exp_capacity}

We begin by empirically validating the expressivity advantage predicted by Theorem~\ref{thm:approx_rate} on a suite of low-dimensional tasks where the ground-truth decision boundary or target function is known analytically. In each experiment, we compare a DBN against a ReLU network of \emph{identical architecture} (same width $W$ and depth $L$), isolating the effect of the activation function.

\paragraph{Classification.}
We consider three 2D binary classification problems of increasing geometric complexity: \textbf{XOR~Blobs}, where four Gaussian clusters are arranged in a checkerboard pattern; \textbf{Two~Moons}, where two interleaving crescent-shaped distributions must be separated; and \textbf{Sinusoidal~Bands}, where the decision boundary follows a sinusoidal curve.

Figure~\ref{fig:activations_classif} presents the learned decision boundaries. The contrast is most striking on the XOR~Blobs task: a single DBN neuron ($W{=}1$, $n{=}3$) achieves perfect classification by learning a parabolic activation (Figure~\ref{fig:activations_classif}a), which partitions the input space into two disjoint regions along the diagonal---precisely the structure the XOR pattern demands. A single ReLU neuron, being monotonic, can only produce a linear decision boundary and is fundamentally unable to separate the four clusters, reaching only $76.5\%$ accuracy.
On Two~Moons ($W{=}2$, $n{=}3$), DBN produces a smooth, curved boundary that tightly envelops the crescent shapes ($100\%$ accuracy), whereas the ReLU baseline with identical width is limited to piecewise-linear separators ($90.8\%$).
The Sinusoidal~Bands task ($W{=}2$, $n{=}5$) further illustrates this advantage: DBN traces the sinusoidal decision boundary and achieves $96.0\%$ accuracy, while the ReLU network approximates it with a coarse piecewise-linear wedge ($91.8\%$).

\paragraph{Regression.}
To probe representational power beyond classification, we fit two analytic target functions: a \textbf{Damped~Oscillator} $y(x) = e^{-0.25x}\cos(x)$ and a \textbf{Gaussian~Sine} $y(x) = e^{-x^2}\sin(5x)$, both of which combine smooth envelope modulation with oscillatory fine structure. Results are shown in Appendix~\ref{sec:regression}.

\begin{figure}[t]
    \centering

    \begin{subfigure}[b]{\textwidth}
        \centering
        \includegraphics[width=0.32\textwidth]{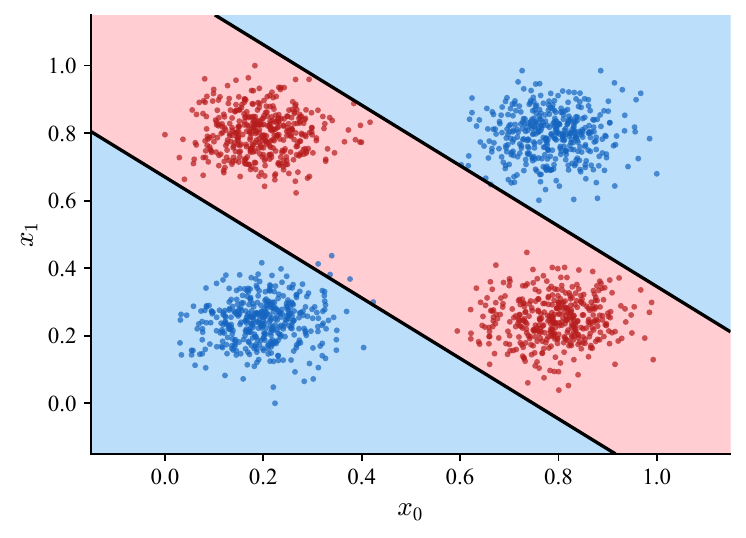}
        \hfill
        \includegraphics[width=0.32\textwidth]{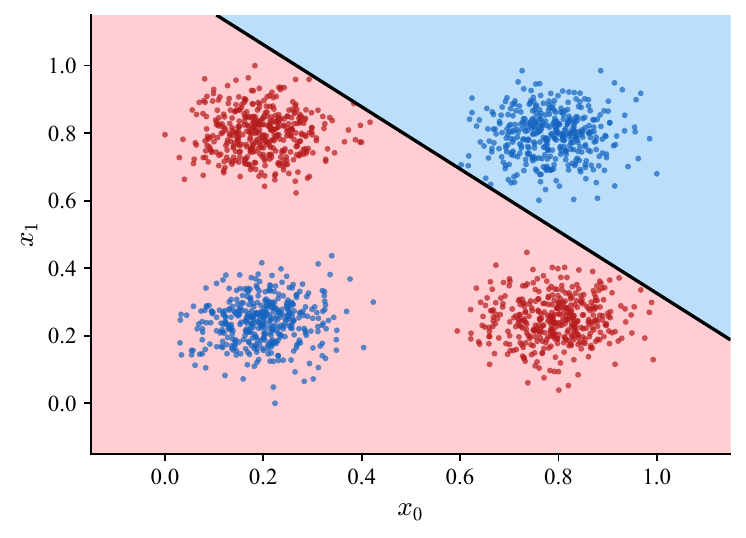}
        \hfill
        \includegraphics[width=0.32\textwidth]{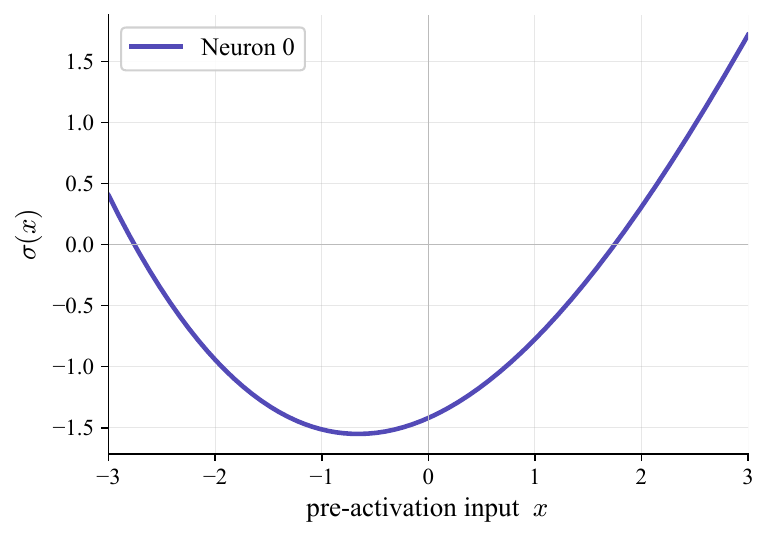}
        \caption{XOR Blobs --- DBN: $W{=}1$, $n{=}3$ (Acc.\ $100\%$); ReLU: $W{=}1$ (Acc.\ $76.5\%$). The single neuron learns a parabolic activation, enabling nonlinear separation of the XOR pattern.}
        \label{fig:row_xor}
    \end{subfigure}
    \\[4pt]

    \begin{subfigure}[b]{\textwidth}
        \centering
        \includegraphics[width=0.28\textwidth]{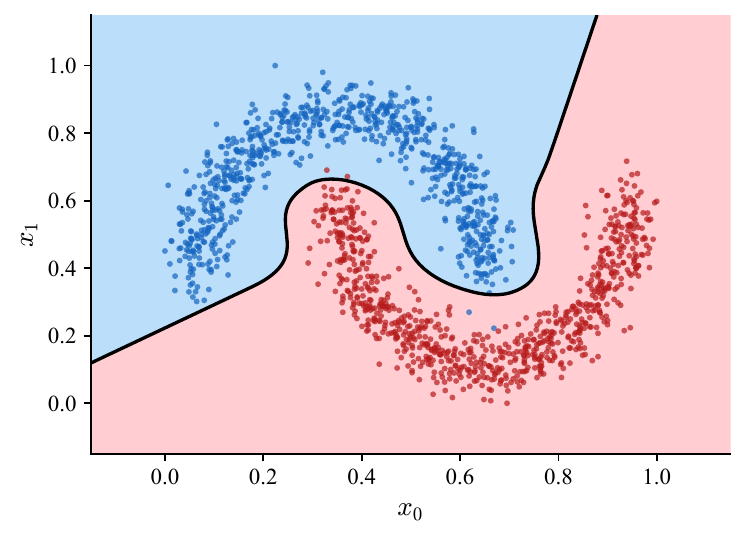}
        \hfill
        \includegraphics[width=0.28\textwidth]{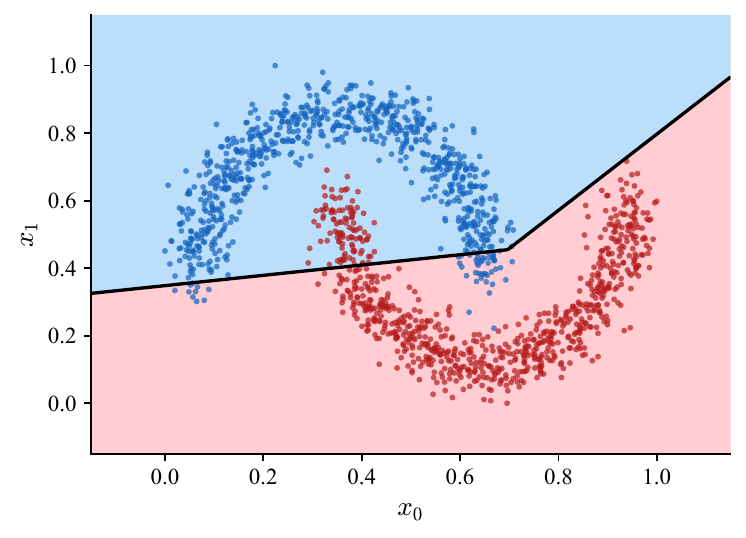}
        \hfill
        \includegraphics[width=0.28\textwidth]{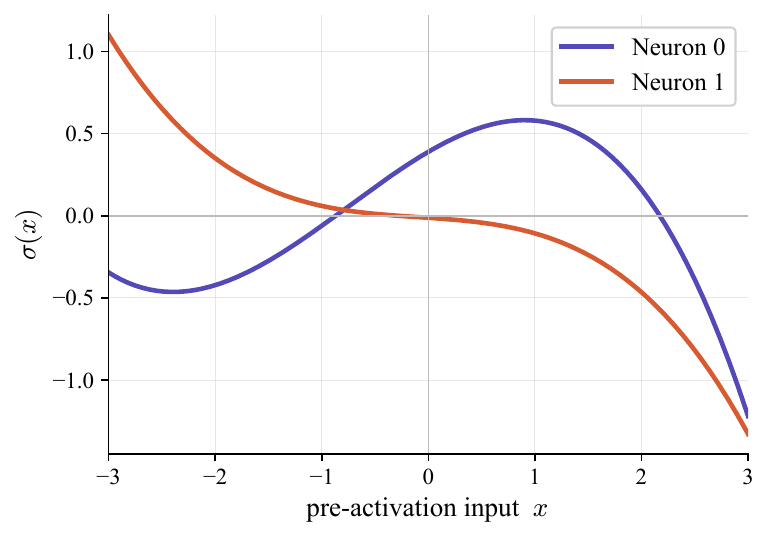}
        \caption{Two Moons --- DBN: $W{=}2$, $n{=}3$ (Acc.\ $100\%$); ReLU: $W{=}2$ (Acc.\ $90.8\%$). Two neurons develop complementary S-shaped curves that jointly carve the crescent boundary.}
        \label{fig:row_moons}
    \end{subfigure}
    \\[4pt]

    \begin{subfigure}[b]{\textwidth}
        \centering
        \includegraphics[width=0.28\textwidth]{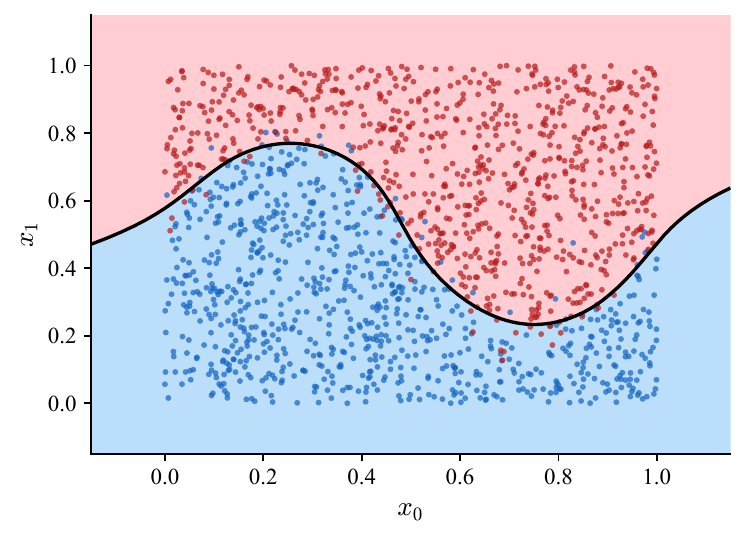}
        \hfill
        \includegraphics[width=0.28\textwidth]{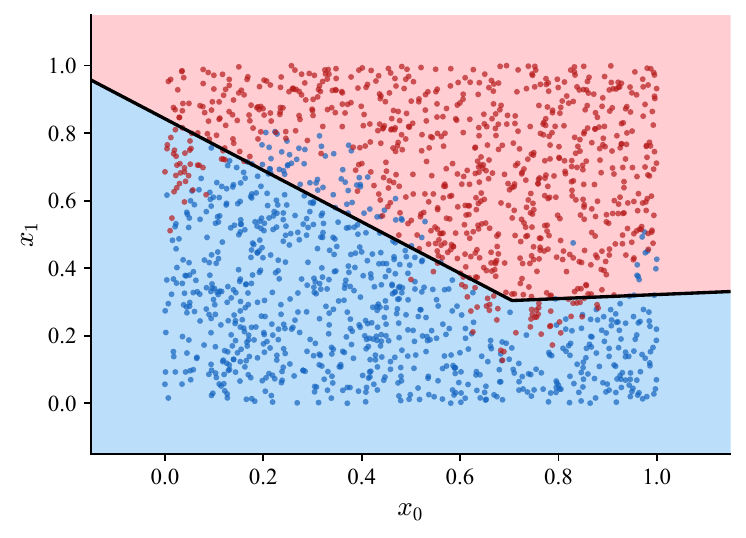}
        \hfill
        \includegraphics[width=0.28\textwidth]{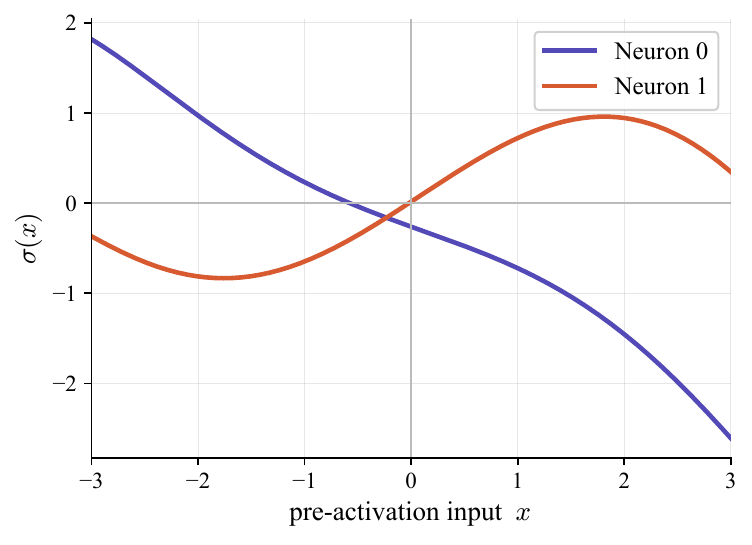}
        \caption{Sinusoidal Bands --- DBN: $W{=}2$, $n{=}5$ (Acc.\ $96.0\%$); ReLU: $W{=}2$ (Acc.\ $91.8\%$). Higher-degree activations ($n{=}5$) develop the curvature needed to trace the sinusoidal boundary.}
        \label{fig:row_sinusoidal}
    \end{subfigure}

    \caption{\textbf{Decision boundaries and learned Bernstein activations on 2D classification tasks.}
    Each row shows the DBN decision boundary (left), ReLU decision boundary of identical width $W$ (center), and the learned Bernstein activation $\sigma(x)$ of the DBN hidden neurons (right).
    DBN activations produce smooth, curved boundaries that match the ground-truth geometry, whereas ReLU networks are restricted to piecewise-linear separators.
    The advantage is most pronounced on XOR~Blobs, where a single DBN neuron learns a parabolic activation that creates two disjoint decision regions---a structure fundamentally inaccessible to a single monotonic ReLU neuron.}
    \label{fig:activations_classif}
\end{figure}

\subsection{Experiment 2: Parameter Efficiency}
\label{sec:exp_compression}

We now move from controlled low-dimensional settings to large-scale binary classification on two established particle physics benchmarks: \textbf{HIGGS}~\cite{baldi2014searching} (11M samples, 28 features) and \textbf{SUSY}~\cite{baldi2014searching} (5M samples, 18 features). Both tasks require distinguishing rare signal events from dominant background processes and are widely used as stress tests for network capacity, since the underlying decision boundaries involve complex, non-linear interactions among kinematic features.

\paragraph{Experimental protocol.}
For each dataset, we train networks across a grid of widths $W \in \{16, 32, 128, 256\}$ and depths $L \in \{5, 15, 50, 100, 150\}$ with residual connections, yielding $4 \times 5 = 20$ architectural configurations. Without residual connections, we restrict to $L \in \{5, 15\}$ ($4 \times 2 = 8$ configurations), since deeper plain networks with standard activations are known to suffer from vanishing gradients that prevent meaningful convergence. Each configuration is run with three random seeds. For ReLU baselines this produces $60$ models (residual) and $24$ models (non-residual) per dataset. For DBNs we additionally sweep degrees $n \in \{3, 9, 15, 25\}$, yielding $240$ models (residual) and $96$ models (non-residual) per dataset. In total, this amounts to $420$ trained models per dataset and $840$ experiments across both benchmarks.

\paragraph{Evaluation methodology.}
For each architecture ($L{\times}W$), we select the best-performing ReLU model across all three seeds as the baseline. We then identify all DBN models---across any architecture and degree---that achieve a strictly better metric (higher AUC or lower loss) than this baseline. Among these qualifying models, we report the one with the fewest parameters as the ``Best DBN.'' The parameter efficiency is defined as $(1 - \text{\#params}_{\text{DBN}} / \text{\#params}_{\text{ReLU}}) \times 100\%$; positive values indicate that the DBN achieves better performance with fewer parameters, while negative values indicate the best-qualifying DBN requires more parameters than the ReLU baseline (which occurs exclusively at the smallest architectures where Bernstein coefficient overhead is proportionally large).


\paragraph{Parameter efficiency scales with architecture size.}
Tables~\ref{tab:higgs_residual_auc_bucket_report} and~\ref{tab:susy_residual_auc_bucket_report} reveal a clear and consistent trend: parameter efficiency increases monotonically with both depth $L$ and width $W$. At the smallest architecture ($5{\times}16$), the efficiency is negative on both datasets ($-20.2\%$ on HIGGS, $-22.4\%$ on SUSY), reflecting the overhead of Bernstein coefficients relative to a tiny weight matrix. However, this overhead is quickly amortized: on SUSY, the efficiency already reaches $64.4\%$ at $5{\times}32$ and $97.5\%$ at $5{\times}128$. On HIGGS, the crossover to positive efficiency occurs at $5{\times}256$ ($73.2\%$).

As depth increases, the pattern strengthens dramatically. At $L{=}50$ and above, parameter efficiency exceeds $85\%$ across all widths on both datasets. On SUSY at $150{\times}256$, the efficiency reaches $\mathbf{100.0\%}$: a DBN with just $2{,}226$ parameters ($5{\times}16$, $n{=}9$) surpasses a $9.8$M-parameter ReLU network. On HIGGS, the maximum efficiency is $99.3\%$ at the same architecture. This scaling behavior is a direct consequence of Theorem~\ref{thm:approx_rate}: increasing depth $L$ causes the effective polynomial degree $n^L$ to grow exponentially, so a shallow DBN can match the representational capacity that a deep ReLU network achieves only through a large number of piecewise-linear partitions.

\paragraph{Deeper ReLU networks are most efficiently replaced.}
A striking observation is that increasing the ReLU network's depth does not proportionally improve its performance, but it does proportionally increase its parameter count---making it increasingly easy for a compact DBN to outperform it. For example, on HIGGS (Table~\ref{tab:higgs_residual_auc_bucket_report}), both $50{\times}256$ (3.2M params, AUC $0.8671$) and $150{\times}256$ (9.8M params, AUC $0.8692$) are surpassed by the same DBN: a $5{\times}128$, $n{=}3$ model with $72{,}578$ parameters. The tripling of the ReLU's parameter count yields only a marginal AUC improvement ($+0.002$), while the matching DBN remains unchanged. This suggests that deep ReLU networks allocate a large fraction of their capacity to emulating non-linear structure that Bernstein activations represent natively.

\paragraph{SUSY vs.\ HIGGS.}
The efficiency gains are consistently higher on SUSY than on HIGGS. On SUSY, efficiencies exceed $95\%$ at $L{\geq}50$ for all widths, whereas on HIGGS, comparable efficiencies require $L{\geq}100$. This is consistent with SUSY being the lower-dimensional task (18 vs.\ 28 features): the target function is smoother relative to the input space, allowing DBNs to approximate it with smaller architectures.

\paragraph{Qualifying model counts.}
The ``Total qual.'' column reveals the breadth of the DBN advantage: for most architectures, the majority of all $240$ DBN models outperform the best ReLU, and all four degree variants contribute qualifying models. Even $n{=}3$ models qualify abundantly, confirming that the advantage is not driven solely by high polynomial degree but by the fundamental expressivity of learnable polynomial activations.

The complete results for the remaining six dataset--metric--residual combinations (HIGGS and SUSY, training loss and AUC, with and without residual connections) are presented in Appendix~\ref{app:bucket_tables}. The same scaling patterns hold across all settings: negative efficiency at the smallest architectures, monotonic increase with depth and width, and efficiencies consistently above $90\%$ for deeper networks.


\begin{table*}[t]
\centering
\caption{\textbf{Parameter efficiency on HIGGS (residual connections, validation AUC).} For each ReLU architecture ($L{\times}W$), the best AUC across seeds serves as the baseline. The best DBN is the fewest-parameter model (any architecture/degree) that strictly exceeds this AUC. Efficiency increases monotonically with depth and width, from $-20.2\%$ at $5{\times}16$ to $99.3\%$ at $150{\times}256$, where a $72{,}578$-parameter DBN surpasses a $9.8$M-parameter ReLU. Negative values at the smallest architectures reflect Bernstein coefficient overhead.}
\label{tab:higgs_residual_auc_bucket_report}
\resizebox{\textwidth}{!}{%
\setlength{\tabcolsep}{5pt}
\begin{tabular}{l r r r
    >{\centering\arraybackslash}p{2.2em}
    >{\centering\arraybackslash}p{2.2em}
    >{\centering\arraybackslash}p{2.2em}
    >{\centering\arraybackslash}p{2.2em} l r r}
\toprule
\multirow{2}{*}{\shortstack{Arch\\($L\times W$)}}  & \multicolumn{2}{c}{ReLU baseline}  & \multirow{2}{*}{\shortstack{Total\\qual.}}  & \multicolumn{4}{c}{Qualifying count by DBN degree}  & \multicolumn{2}{c}{Best DBN}  & \multirow{2}{*}{\shortstack{Param\\effic.}} \\
\cmidrule(lr){2-3} \cmidrule(lr){4-7} \cmidrule(lr){8-9}
  & \# params & auc  & & $n{=}3$ & $n{=}9$ & $n{=}15$ & $n{=}25$  & \# params (arch.,~$n$) & auc  & \\
\midrule
  $5\times16$ & 1,586 & 0.8209 & 240 & 60 & 60 & 60 & 60 & 1,906~~($5\times16$,~$n{=}3$) & 0.8288 & $-20.2\%$ \\
  $5\times32$ & 5,218 & 0.8395 & 214 & 54 & 55 & 57 & 48 & 5,858~~($5\times32$,~$n{=}3$) & 0.8451 & $-12.3\%$ \\
  $5\times128$ & 70,018 & 0.8638 & 109 & 19 & 30 & 30 & 30 & 72,578~~($5\times128$,~$n{=}3$) & 0.8729 & $-3.7\%$ \\
  $5\times256$ & 271,106 & 0.8692 & 79 & 5 & 18 & 28 & 28 & 72,578~~($5\times128$,~$n{=}3$) & 0.8729 & $\mathbf{73.2\%}$ \\
\midrule
  $15\times16$ & 4,306 & 0.8217 & 240 & 60 & 60 & 60 & 60 & 1,906~~($5\times16$,~$n{=}3$) & 0.8288 & $\mathbf{55.7\%}$ \\
  $15\times32$ & 15,778 & 0.8450 & 198 & 47 & 53 & 53 & 45 & 5,858~~($5\times32$,~$n{=}3$) & 0.8451 & $\mathbf{62.9\%}$ \\
  $15\times128$ & 235,138 & 0.8658 & 99 & 9 & 30 & 30 & 30 & 72,578~~($5\times128$,~$n{=}3$) & 0.8729 & $\mathbf{69.1\%}$ \\
  $15\times256$ & 929,026 & 0.8675 & 89 & 8 & 23 & 29 & 29 & 72,578~~($5\times128$,~$n{=}3$) & 0.8729 & $\mathbf{92.2\%}$ \\
\midrule
  $50\times16$ & 13,826 & 0.8268 & 239 & 59 & 60 & 60 & 60 & 1,906~~($5\times16$,~$n{=}3$) & 0.8288 & $\mathbf{86.2\%}$ \\
  $50\times32$ & 52,738 & 0.8443 & 200 & 48 & 53 & 53 & 46 & 5,858~~($5\times32$,~$n{=}3$) & 0.8451 & $\mathbf{88.9\%}$ \\
  $50\times128$ & 813,058 & 0.8624 & 123 & 31 & 31 & 31 & 30 & 59,138~~($50\times32$,~$n{=}3$) & 0.8625 & $\mathbf{92.7\%}$ \\
  $50\times256$ & 3,231,746 & 0.8671 & 95 & 9 & 27 & 30 & 29 & 72,578~~($5\times128$,~$n{=}3$) & 0.8729 & $\mathbf{97.8\%}$ \\
\midrule
  $100\times16$ & 27,426 & 0.8266 & 239 & 59 & 60 & 60 & 60 & 1,906~~($5\times16$,~$n{=}3$) & 0.8288 & $\mathbf{93.1\%}$ \\
  $100\times32$ & 105,538 & 0.8482 & 180 & 41 & 48 & 46 & 45 & 6,818~~($5\times32$,~$n{=}9$) & 0.8496 & $\mathbf{93.5\%}$ \\
  $100\times128$ & 1,638,658 & 0.8629 & 117 & 26 & 30 & 31 & 30 & 72,578~~($5\times128$,~$n{=}3$) & 0.8729 & $\mathbf{95.6\%}$ \\
  $100\times256$ & 6,521,346 & 0.8687 & 83 & 6 & 21 & 28 & 28 & 72,578~~($5\times128$,~$n{=}3$) & 0.8729 & $\mathbf{98.9\%}$ \\
\midrule
  $150\times16$ & 41,026 & 0.8263 & 239 & 59 & 60 & 60 & 60 & 1,906~~($5\times16$,~$n{=}3$) & 0.8288 & $\mathbf{95.4\%}$ \\
  $150\times32$ & 158,338 & 0.8469 & 187 & 44 & 50 & 48 & 45 & 6,818~~($5\times32$,~$n{=}9$) & 0.8496 & $\mathbf{95.7\%}$ \\
  $150\times128$ & 2,464,258 & 0.8637 & 111 & 21 & 30 & 30 & 30 & 72,578~~($5\times128$,~$n{=}3$) & 0.8729 & $\mathbf{97.1\%}$ \\
  $150\times256$ & 9,810,946 & 0.8692 & 79 & 5 & 18 & 28 & 28 & 72,578~~($5\times128$,~$n{=}3$) & 0.8729 & $\mathbf{99.3\%}$ \\
\bottomrule
\end{tabular}%
}
\end{table*}

\begin{table*}[t]
\centering
\caption{\textbf{Parameter efficiency on SUSY (residual connections, validation AUC).} Same methodology as Table~\ref{tab:higgs_residual_auc_bucket_report}. Efficiency reaches $\mathbf{100.0\%}$ at $150{\times}256$: a DBN with $2{,}226$ parameters ($5{\times}16$, $n{=}9$) surpasses a $9.8$M-parameter ReLU network. The crossover to positive efficiency occurs already at $5{\times}32$ ($64.4\%$), and efficiencies exceed $95\%$ for all architectures with $L \geq 50$.}
\label{tab:susy_residual_auc_bucket_report}
\resizebox{\textwidth}{!}{%
\setlength{\tabcolsep}{5pt}
\begin{tabular}{l r r r
    >{\centering\arraybackslash}p{2.2em}
    >{\centering\arraybackslash}p{2.2em}
    >{\centering\arraybackslash}p{2.2em}
    >{\centering\arraybackslash}p{2.2em} l r r}
\toprule
\multirow{2}{*}{\shortstack{Arch\\($L\times W$)}}  & \multicolumn{2}{c}{ReLU baseline}  & \multirow{2}{*}{\shortstack{Total\\qual.}}  & \multicolumn{4}{c}{Qualifying count by DBN degree}  & \multicolumn{2}{c}{Best DBN}  & \multirow{2}{*}{\shortstack{Param\\effic.}} \\
\cmidrule(lr){2-3} \cmidrule(lr){4-7} \cmidrule(lr){8-9}
  & \# params & auc  & & $n{=}3$ & $n{=}9$ & $n{=}15$ & $n{=}25$  & \# params (arch.,~$n$) & auc  & \\
\midrule
  $5\times16$ & 1,426 & 0.8759 & 235 & 56 & 60 & 60 & 59 & 1,746~~($5\times16$,~$n{=}3$) & 0.8770 & $-22.4\%$ \\
  $5\times32$ & 4,898 & 0.8767 & 193 & 51 & 57 & 52 & 33 & 1,746~~($5\times16$,~$n{=}3$) & 0.8770 & $\mathbf{64.4\%}$ \\
  $5\times128$ & 68,738 & 0.8769 & 183 & 51 & 57 & 46 & 29 & 1,746~~($5\times16$,~$n{=}3$) & 0.8770 & $\mathbf{97.5\%}$ \\
  $5\times256$ & 268,546 & 0.8774 & 116 & 39 & 39 & 24 & 14 & 2,226~~($5\times16$,~$n{=}9$) & 0.8775 & $\mathbf{99.2\%}$ \\
\midrule
  $15\times16$ & 4,146 & 0.8758 & 235 & 56 & 60 & 60 & 59 & 1,746~~($5\times16$,~$n{=}3$) & 0.8770 & $\mathbf{57.9\%}$ \\
  $15\times32$ & 15,458 & 0.8770 & 173 & 47 & 55 & 44 & 27 & 1,746~~($5\times16$,~$n{=}3$) & 0.8770 & $\mathbf{88.7\%}$ \\
  $15\times128$ & 233,858 & 0.8769 & 182 & 51 & 57 & 45 & 29 & 1,746~~($5\times16$,~$n{=}3$) & 0.8770 & $\mathbf{99.3\%}$ \\
  $15\times256$ & 926,466 & 0.8771 & 154 & 44 & 52 & 37 & 21 & 2,226~~($5\times16$,~$n{=}9$) & 0.8775 & $\mathbf{99.8\%}$ \\
\midrule
  $50\times16$ & 13,666 & 0.8764 & 211 & 53 & 57 & 58 & 43 & 1,746~~($5\times16$,~$n{=}3$) & 0.8770 & $\mathbf{87.2\%}$ \\
  $50\times32$ & 52,418 & 0.8770 & 164 & 46 & 52 & 41 & 25 & 2,226~~($5\times16$,~$n{=}9$) & 0.8775 & $\mathbf{95.8\%}$ \\
  $50\times128$ & 811,778 & 0.8773 & 136 & 41 & 46 & 32 & 17 & 2,226~~($5\times16$,~$n{=}9$) & 0.8775 & $\mathbf{99.7\%}$ \\
  $50\times256$ & 3,229,186 & 0.8770 & 169 & 46 & 54 & 43 & 26 & 2,226~~($5\times16$,~$n{=}9$) & 0.8775 & $\mathbf{99.9\%}$ \\
\midrule
  $100\times16$ & 27,266 & 0.8761 & 226 & 55 & 58 & 60 & 53 & 1,746~~($5\times16$,~$n{=}3$) & 0.8770 & $\mathbf{93.6\%}$ \\
  $100\times32$ & 105,218 & 0.8764 & 214 & 53 & 57 & 59 & 45 & 1,746~~($5\times16$,~$n{=}3$) & 0.8770 & $\mathbf{98.3\%}$ \\
  $100\times128$ & 1,637,378 & 0.8771 & 159 & 45 & 52 & 39 & 23 & 2,226~~($5\times16$,~$n{=}9$) & 0.8775 & $\mathbf{99.9\%}$ \\
  $100\times256$ & 6,518,786 & 0.8778 & 59 & 22 & 22 & 9 & 6 & 5,538~~($5\times32$,~$n{=}3$) & 0.8779 & $\mathbf{99.9\%}$ \\
\midrule
  $150\times16$ & 40,866 & 0.8770 & 173 & 47 & 55 & 44 & 27 & 1,746~~($5\times16$,~$n{=}3$) & 0.8770 & $\mathbf{95.7\%}$ \\
  $150\times32$ & 158,018 & 0.8762 & 222 & 54 & 57 & 60 & 51 & 1,746~~($5\times16$,~$n{=}3$) & 0.8770 & $\mathbf{98.9\%}$ \\
  $150\times128$ & 2,462,978 & 0.8770 & 165 & 46 & 52 & 41 & 26 & 2,226~~($5\times16$,~$n{=}9$) & 0.8775 & $\mathbf{99.9\%}$ \\
  $150\times256$ & 9,808,386 & 0.8773 & 141 & 41 & 48 & 34 & 18 & 2,226~~($5\times16$,~$n{=}9$) & 0.8775 & $\mathbf{100.0\%}$ \\
\bottomrule
\end{tabular}%
}
\end{table*}

\subsection{Experiment 3: Convergence Speed}
\label{sec:exp_convergence}

We now examine the training dynamics of DBNs relative to ReLU baselines. For every architecture in our experimental grid (\S\ref{sec:exp_compression}), we compute two quantities from the per-epoch training loss curves (median over three seeds): (i)~the \emph{crossover epoch}, defined as the first epoch at which the DBN's median loss falls below the ReLU's final (best) median loss, and (ii)~the \emph{loss improvement} at the end of training, defined as $(\mathcal{L}_{\text{DBN}} - \mathcal{L}_{\text{ReLU}}) / \mathcal{L}_{\text{ReLU}} \times 100\%$, where $\mathcal{L}$ denotes the median loss at the final epoch. The total epoch count reported for each model is the median across the three seeds.

\paragraph{DBNs reach ReLU's final loss early in training.}
Tables~\ref{tab:crossover_higgs_res} and~\ref{tab:crossover_susy_res} present the crossover analysis for HIGGS and SUSY with residual connections. Across nearly all architectures and degrees, DBNs surpass ReLU's best training loss well before their own training concludes. On HIGGS (Table~\ref{tab:crossover_higgs_res}), the crossover epoch is typically $2{\times}$--$6{\times}$ earlier than the ReLU's total training duration. For example, at $150{\times}256$, ReLU trains for 55 epochs while the $n{=}25$ DBN crosses over at epoch~13---requiring only $24\%$ of ReLU's training budget to match its final performance. On SUSY (Table~\ref{tab:crossover_susy_res}), at $150{\times}256$, the $n{=}25$ DBN crosses over at epoch~8 out 
of ReLU's 40-epoch run.

\paragraph{Higher degree accelerates convergence.}
Within each architecture, higher Bernstein degrees consistently achieve earlier crossover. On HIGGS at $50{\times}128$, the crossover occurs at epoch 29 for $n{=}3$, epoch 19 for $n{=}9$, epoch 15 for $n{=}15$, and epoch 13 for $n{=}25$. This pattern is universal across both datasets and all architectures, and is consistent with the theoretical prediction: higher $n$ increases the effective polynomial degree $n^L$, enabling the network to fit the target function with fewer optimization steps.

\paragraph{Loss improvement scales with depth and degree.}
The loss improvement at the end of training reveals a strong interaction between depth and degree. On HIGGS at $L{=}5$, the improvements are modest (1.7\%--8.2\%). At $L{=}50$, they grow substantially, reaching $38.3\%$ at $50{\times}256$ with $n{=}25$. At $L{=}150$, the $n{=}25$ DBN achieves a $\mathbf{45.1\%}$ loss improvement over ReLU at $150{\times}256$. This scaling is predicted by Theorem~\ref{thm:approx_rate}: the approximation error decays as $n^{-L}$, so increasing both $n$ and $L$ yields compounding gains.


\paragraph{Failure cases.}
The few entries marked ``--'' indicate architectures where $n{=}3$ DBNs did not surpass ReLU's final loss (e.g., $100{\times}128$ on HIGGS). These occur exclusively at $n{=}3$ with large width, where the cubic polynomial lacks sufficient curvature to outperform a well-optimized deep ReLU network at equal width. Higher degrees ($n \geq 9$) cross over successfully in all cases.

The corresponding results for non-residual configurations and wall-clock timing analysis are provided in Appendix~\ref{app:convergence} and Appendix~\ref{app:timing}, respectively.

\begin{table}[t]
\centering
\setlength{\tabcolsep}{3pt}
\caption{\textbf{Convergence analysis on HIGGS (residual connections).} For each architecture, the ReLU column reports the median total training epochs across seeds. Each DBN cell shows \emph{crossover\,/\,total epochs} followed by the loss improvement at the final epoch. Crossover is the first epoch at which the DBN's median training loss (over 3 seeds) falls below ReLU's final median loss. Loss improvement $= (\mathcal{L}_{\text{DBN}} - \mathcal{L}_{\text{ReLU}}) / \mathcal{L}_{\text{ReLU}} \times 100\%$. ``--'' indicates the DBN did not surpass ReLU's final loss. Improvements scale with both depth $L$ and degree $n$, reaching $45.1\%$ at $150{\times}256$ with $n{=}25$.}
\label{tab:crossover_higgs_res}
\footnotesize{
\begin{tabular}{lccccc}
\toprule
\textbf{Arch} & \textbf{ReLU} & \multicolumn{4}{c}{\textbf{DBN}} \\
\cmidrule(lr){3-6}
($L{\times}W$) & (epochs) & $n=3$ & $n=9$ & $n=15$ & $n=25$ \\
\midrule
5$\times$16 & 72 & 26/85 ($\downarrow$2.0\%) & 14/81 ($\downarrow$3.0\%) & 12/66 ($\downarrow$3.5\%) & 14/69 ($\downarrow$3.2\%) \\
5$\times$32 & 73 & 32/93 ($\downarrow$1.7\%) & 18/95 ($\downarrow$3.0\%) & 14/88 ($\downarrow$3.4\%) & 12/81 ($\downarrow$3.9\%) \\
5$\times$128 & 68 & 37/111 ($\downarrow$3.0\%) & 20/98 ($\downarrow$6.2\%) & 16/116 ($\downarrow$8.2\%) & 12/91 ($\downarrow$7.4\%) \\
5$\times$256 & 71 & 51/73 ($\downarrow$3.2\%) & 27/48 ($\downarrow$3.4\%) & 20/36 ($\downarrow$3.2\%) & 15/43 ($\downarrow$7.1\%) \\
\midrule
15$\times$16 & 75 & 12/71 ($\downarrow$4.1\%) & 9/58 ($\downarrow$4.9\%) & 11/79 ($\downarrow$5.3\%) & 13/94 ($\downarrow$4.5\%) \\
15$\times$32 & 84 & 21/87 ($\downarrow$3.7\%) & 14/96 ($\downarrow$4.0\%) & 14/67 ($\downarrow$3.9\%) & 15/84 ($\downarrow$4.0\%) \\
15$\times$128 & 70 & 58/73 ($\downarrow$2.7\%) & 36/48 ($\downarrow$1.5\%) & 23/49 ($\downarrow$3.9\%) & 22/45 ($\downarrow$4.6\%) \\
15$\times$256 & 40 & 30/32 ($\downarrow$2.3\%) & 23/25 ($\downarrow$4.6\%) & 20/26 ($\downarrow$9.3\%) & 16/24 ($\downarrow$11.8\%) \\
\midrule
50$\times$16 & 108 & 13/86 ($\downarrow$5.6\%) & 12/98 ($\downarrow$6.9\%) & 11/91 ($\downarrow$6.0\%) & 18/90 ($\downarrow$4.1\%) \\
50$\times$32 & 80 & 14/94 ($\downarrow$5.6\%) & 11/74 ($\downarrow$5.7\%) & 10/90 ($\downarrow$6.2\%) & 15/70 ($\downarrow$4.7\%) \\
50$\times$128 & 54 & 29/36 ($\downarrow$4.4\%) & 19/33 ($\downarrow$11.0\%) & 15/29 ($\downarrow$13.7\%) & 13/26 ($\downarrow$14.0\%) \\
50$\times$256 & 37 & 25/29 ($\downarrow$9.6\%) & 18/25 ($\downarrow$20.9\%) & 14/23 ($\downarrow$30.1\%) & 12/22 ($\downarrow$38.3\%) \\
\midrule
100$\times$16 & 92 & 12/64 ($\downarrow$5.6\%) & 10/71 ($\downarrow$6.5\%) & 12/104 ($\downarrow$7.3\%) & 18/94 ($\downarrow$4.4\%) \\
100$\times$32 & 110 & 17/70 ($\downarrow$4.5\%) & 13/60 ($\downarrow$5.1\%) & 13/60 ($\downarrow$5.3\%) & 16/63 ($\downarrow$4.3\%) \\
100$\times$128 & 62 & --/31 ($\uparrow$0.8\%) & 24/29 ($\downarrow$9.2\%) & 18/26 ($\downarrow$16.5\%) & 17/26 ($\downarrow$20.3\%) \\
100$\times$256 & 54 & 26/30 ($\downarrow$11.0\%) & 20/24 ($\downarrow$21.6\%) & 16/23 ($\downarrow$30.1\%) & 14/22 ($\downarrow$40.3\%) \\
\midrule
150$\times$16 & 92 & 12/86 ($\downarrow$6.1\%) & 11/71 ($\downarrow$6.6\%) & 12/59 ($\downarrow$6.5\%) & 21/90 ($\downarrow$4.2\%) \\
150$\times$32 & 111 & 15/78 ($\downarrow$5.8\%) & 11/64 ($\downarrow$7.2\%) & 11/60 ($\downarrow$8.4\%) & 14/54 ($\downarrow$5.5\%) \\
150$\times$128 & 61 & 28/40 ($\downarrow$11.1\%) & 19/28 ($\downarrow$13.3\%) & 15/26 ($\downarrow$20.0\%) & 14/27 ($\downarrow$28.6\%) \\
150$\times$256 & 55 & 23/30 ($\downarrow$12.6\%) & 19/25 ($\downarrow$23.1\%) & 15/24 ($\downarrow$35.9\%) & 13/23 ($\downarrow$45.1\%) \\
\bottomrule
\end{tabular}
}
\end{table}

\begin{table}[t]
\centering
\setlength{\tabcolsep}{3pt}
\caption{\textbf{Convergence analysis on SUSY (residual connections).} Same methodology as Table~\ref{tab:crossover_higgs_res}. SUSY, being a lower-dimensional and easier task, shows smaller absolute improvements but the same qualitative patterns: earlier crossover at higher degrees, and loss improvement scaling with depth. The maximum improvement is $10.8\%$ at $100{\times}256$ with $n{=}25$. ``--'' entries occur only for $n{=}3$ at large widths.}
\label{tab:crossover_susy_res}
\footnotesize{
\begin{tabular}{lccccc}
\toprule
\textbf{Arch} & \textbf{ReLU} & \multicolumn{4}{c}{\textbf{DBN}} \\
\cmidrule(lr){3-6}
($L{\times}W$) & (epochs) & $n=3$ & $n=9$ & $n=15$ & $n=25$ \\
\midrule
5$\times$16 & 41 & 28/38 ($\downarrow$0.2\%) & 18/45 ($\downarrow$0.5\%) & 17/41 ($\downarrow$0.6\%) & 14/38 ($\downarrow$0.6\%) \\
5$\times$32 & 43 & 32/37 ($\downarrow$0.1\%) & 18/38 ($\downarrow$0.4\%) & 17/42 ($\downarrow$0.4\%) & 14/40 ($\downarrow$0.5\%) \\
5$\times$128 & 37 & --/39 ($\uparrow$0.2\%) & 32/39 ($\downarrow$0.4\%) & 25/27 ($\downarrow$0.1\%) & 20/29 ($\downarrow$0.5\%) \\
5$\times$256 & 30 & 32/39 ($\downarrow$0.3\%) & 21/34 ($\downarrow$1.1\%) & 17/27 ($\downarrow$0.7\%) & 14/23 ($\downarrow$0.9\%) \\
\midrule
15$\times$16 & 38 & 24/44 ($\downarrow$0.4\%) & 21/46 ($\downarrow$0.5\%) & 16/39 ($\downarrow$0.6\%) & 19/44 ($\downarrow$0.8\%) \\
15$\times$32 & 51 & 26/43 ($\downarrow$0.3\%) & 21/42 ($\downarrow$0.4\%) & 24/34 ($\downarrow$0.3\%) & 24/38 ($\downarrow$0.4\%) \\
15$\times$128 & 32 & --/33 ($\uparrow$0.1\%) & 25/31 ($\downarrow$0.3\%) & 21/27 ($\downarrow$1.1\%) & 18/24 ($\downarrow$1.6\%) \\
15$\times$256 & 35 & 29/39 ($\downarrow$0.9\%) & 20/23 ($\downarrow$0.4\%) & 16/27 ($\downarrow$4.2\%) & 14/25 ($\downarrow$6.7\%) \\
\midrule
50$\times$16 & 50 & 30/49 ($\downarrow$0.3\%) & 25/56 ($\downarrow$0.6\%) & 27/51 ($\downarrow$0.7\%) & 27/46 ($\downarrow$0.4\%) \\
50$\times$32 & 39 & 20/47 ($\downarrow$0.5\%) & 20/45 ($\downarrow$0.7\%) & 16/39 ($\downarrow$0.8\%) & 17/39 ($\downarrow$1.0\%) \\
50$\times$128 & 44 & --/30 ($\uparrow$1.1\%) & 32/34 ($\downarrow$1.3\%) & 25/27 ($\downarrow$1.9\%) & 21/23 ($\downarrow$2.1\%) \\
50$\times$256 & 44 & --/31 ($\uparrow$0.3\%) & 25/27 ($\downarrow$1.7\%) & 20/27 ($\downarrow$5.6\%) & 17/25 ($\downarrow$8.6\%) \\
\midrule
100$\times$16 & 66 & 33/37 ($\downarrow$0.1\%) & 25/49 ($\downarrow$0.3\%) & 22/46 ($\downarrow$0.7\%) & 22/49 ($\downarrow$0.8\%) \\
100$\times$32 & 59 & 31/48 ($\downarrow$0.2\%) & 29/45 ($\downarrow$0.3\%) & 25/39 ($\downarrow$0.5\%) & 25/34 ($\downarrow$0.4\%) \\
100$\times$128 & 44 & --/29 ($\uparrow$0.7\%) & 27/33 ($\downarrow$2.0\%) & 23/25 ($\downarrow$1.6\%) & 19/22 ($\downarrow$1.5\%) \\
100$\times$256 & 44 & 19/34 ($\downarrow$1.1\%) & 14/28 ($\downarrow$3.7\%) & 13/26 ($\downarrow$6.1\%) & 11/25 ($\downarrow$10.8\%) \\
\midrule
150$\times$16 & 82 & 37/47 ($\downarrow$0.3\%) & 33/48 ($\downarrow$0.5\%) & 31/46 ($\downarrow$0.4\%) & 27/44 ($\downarrow$0.5\%) \\
150$\times$32 & 50 & 20/47 ($\downarrow$0.6\%) & 18/47 ($\downarrow$1.1\%) & 16/31 ($\downarrow$0.6\%) & 17/32 ($\downarrow$0.6\%) \\
150$\times$128 & 48 & 29/33 ($\downarrow$0.2\%) & 20/34 ($\downarrow$2.5\%) & 17/24 ($\downarrow$1.3\%) & 14/26 ($\downarrow$5.3\%) \\
150$\times$256 & 40 & 10/38 ($\downarrow$2.0\%) & 9/27 ($\downarrow$3.4\%) & 8/27 ($\downarrow$7.8\%) & 8/24 ($\downarrow$9.8\%) \\
\bottomrule
\end{tabular}
}
\end{table}


\subsection{Experiment 4: Distributional Analysis Across Activation Functions}
\label{sec:exp_distributional}

To verify that the advantages observed in \S\ref{sec:exp_compression}--\S\ref{sec:exp_convergence} are not specific to ReLU, we extend the comparison to three additional smooth activations: Leaky ReLU (slope $0.1$), SELU, and GeLU. These baselines are trained on the same architecture grid and seeds, adding $504$ models ($3$ activations $\times$ $84$ configurations $\times$ $2$ datasets) to the existing $840$, for a total of $\mathbf{1{,}344}$ trained models.

Figures~\ref{fig:auc_loss_histograms} present stacked histograms of validation AUC and training loss across all models, with standard activations shown in red tones and DBN variants in blue tones. On both HIGGS and SUSY with residual connections (Figure~\ref{fig:auc_loss_histograms}, top), DBNs consistently occupy the high-AUC tail while all four standard activations---including the smooth alternatives---cluster together at lower AUC values. The separation is particularly clear on HIGGS, where AUC buckets above $0.875$ are dominated by DBNs and the region above $0.8875$ is exclusively occupied by higher-degree variants ($n \geq 9$). Notably, SELU, GeLU, and Leaky ReLU do not meaningfully extend beyond the ReLU distribution, indicating that the DBN advantage stems from the learnable polynomial structure rather than mere smoothness.

For training loss (Figure~\ref{fig:auc_loss_histograms}, bottom), the pattern is mirrored: DBNs dominate the low-loss (left) tail. On HIGGS with residual connections, losses below $0.31$ are achieved only by DBNs with $n \geq 9$, and losses below $0.22$ exclusively by $n{=}25$. The standard smooth activations remain clustered with ReLU throughout.

The non-residual distributions show the same qualitative separation and are presented in Appendix~\ref{app:distributional}.

\begin{figure}[t]
    \centering
    \includegraphics[width=0.49\textwidth]{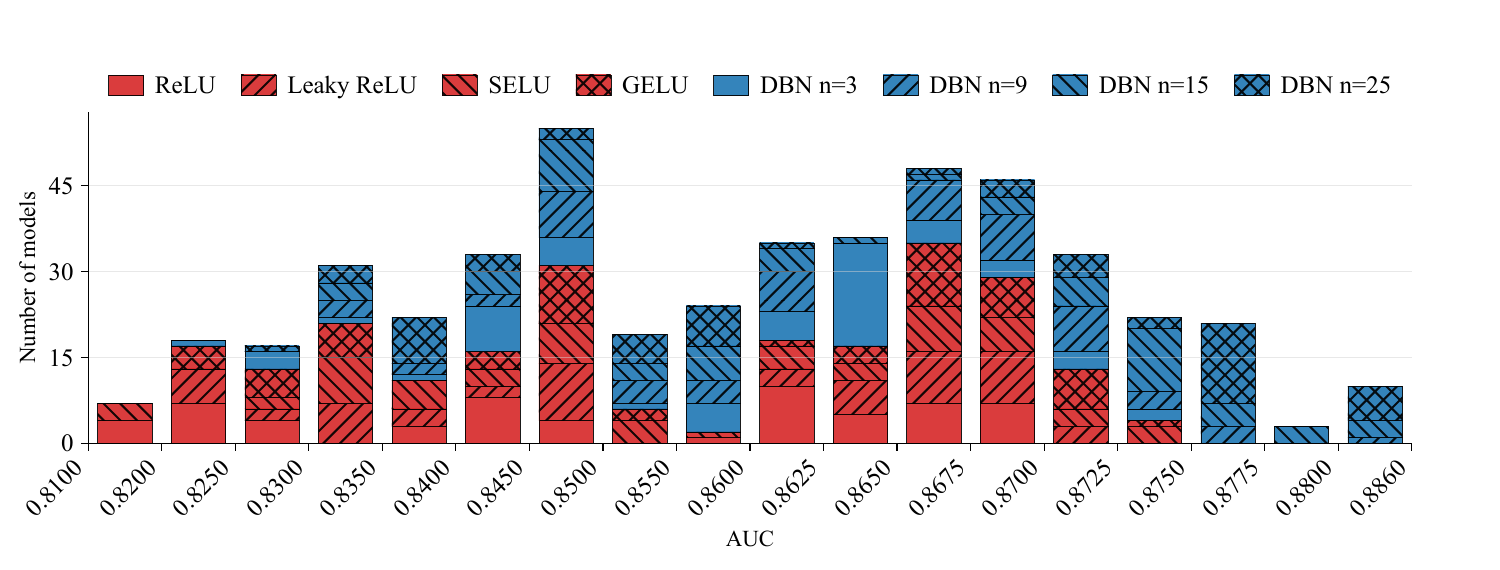}
    \hfill
    \includegraphics[width=0.49\textwidth]{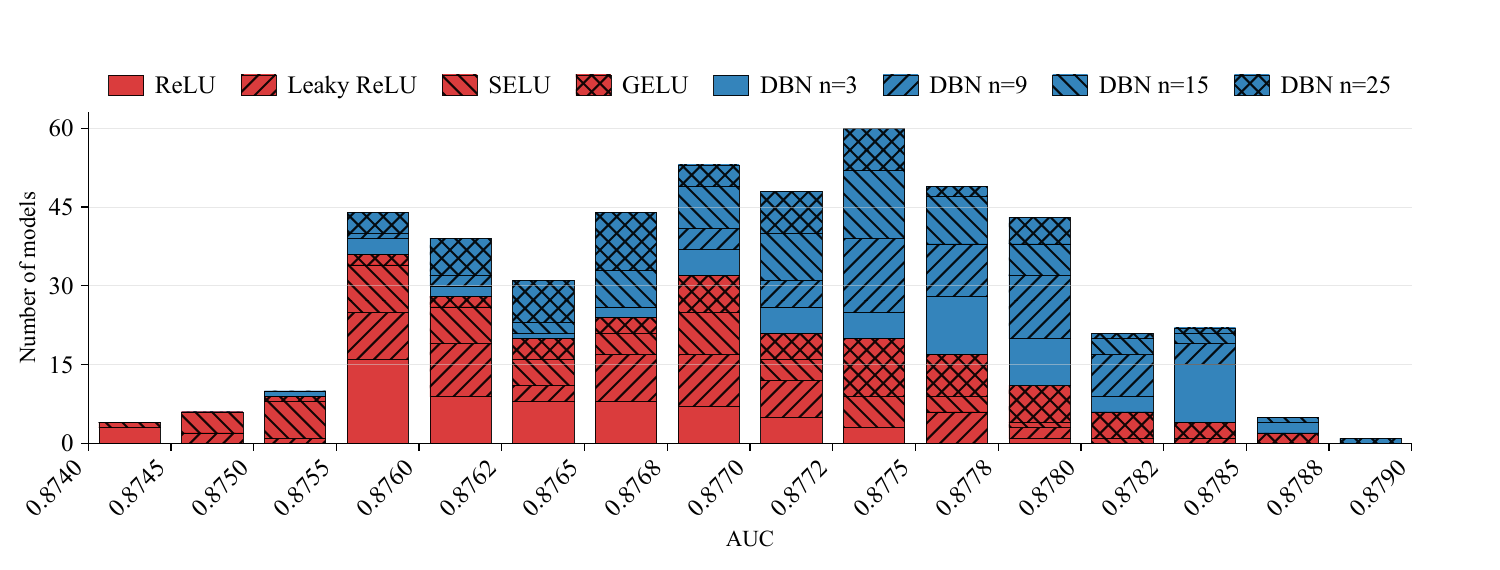}
    \\[4pt]
    \includegraphics[width=0.49\textwidth]{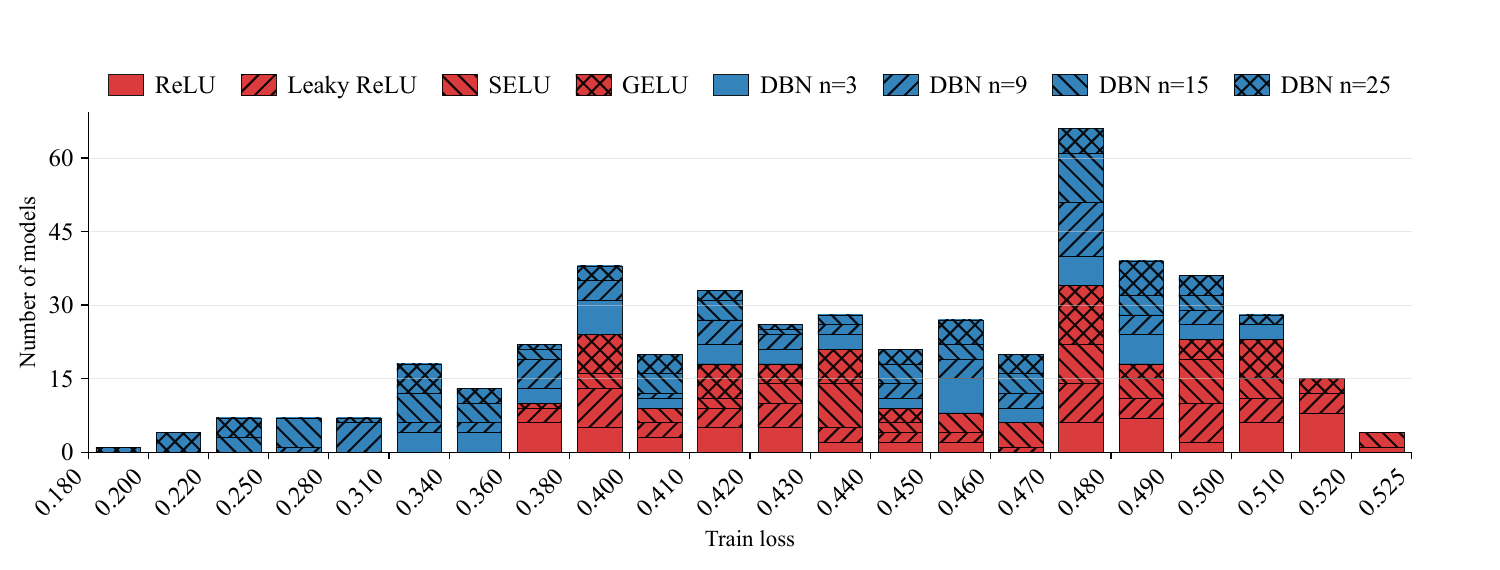}
    \hfill
    \includegraphics[width=0.49\textwidth]{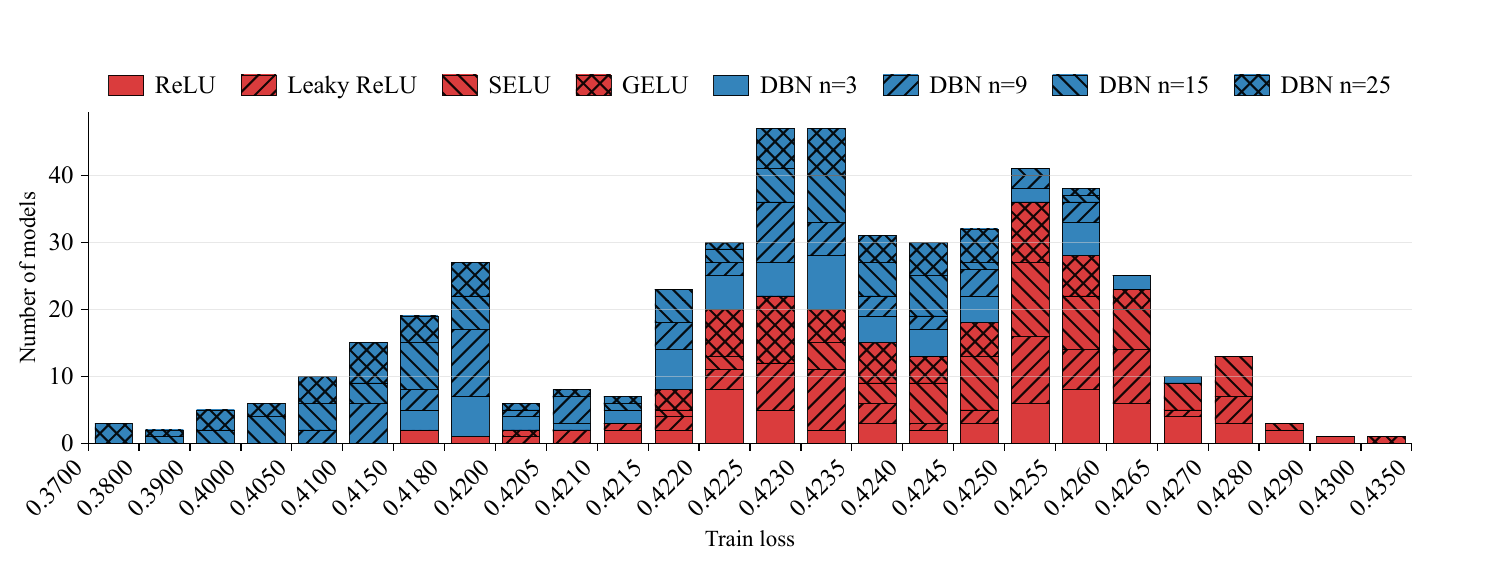}
    \caption{\textbf{Distribution of validation AUC (top) and training loss (bottom, lower is better) with residual connections, on HIGGS (left) and SUSY (right).}
    Stacked histograms over $1{,}344$ models. DBN variants (blue tones) dominate the high-AUC and low-loss tails across both datasets, while standard activations (ReLU, Leaky ReLU, SELU, GeLU; red tones) cluster in the mid-range. On HIGGS, AUC buckets above $0.8775$ and losses below $0.31$ are exclusively higher-degree DBNs ($n \geq 9$); losses below $0.22$ require $n{=}25$.}
    \label{fig:auc_loss_histograms}
\end{figure}

{
\small

\clearpage
\bibliography{bibliography}
\bibliographystyle{abbrvnat}   

}




\appendix

\section{Mathematical Preliminaries}
\subsection{Stable training of DBNs}
\begin{proposition}[\cite{khedr2024deepbern}]\label{app:gradient}Consider the Bernstein activation function $\sigma(x;l,u,\boldsymbol{c})$ of arbitrary order $n$. The following holds:
\begin{enumerate}
        \item $\left\vert \sigma'(x;l,u,c) \right\vert \le 2 n \max_{k \in \{0,\ldots,n\}} \vert c_k \vert$,
        \item $\left\vert \sigma'(x;l,u,c) \right\vert \le 1 \quad$  for all $ i \in \{0, \ldots,  n\}$.
\end{enumerate}
\end{proposition}

Proposition \ref{app:gradient} ensures that the gradient magnitude depends linearly on the learnable coefficients $c_k$, rather than exploding with high powers of the input $x$ (as seen in monomial activations $x^n$ for $x > 1$ \cite{gottemukkula2020polynomial}). This structural upper bound guarantees that training remains stable even in deep regimes. Furthermore, it implies that the Lipschitz constant of the network can be explicitly controlled by applying standard regularization (e.g., $L_2$ weight decay) to the coefficients $c$, directly penalizing large gradients.

\subsection{Preliminaries: Modulus of Continuity}
\label{app:mod_of_cont}
The approximation rate of any function approximator is fundamentally limited by the smoothness of the target function $f$. This is typically quantified using the \textit{modulus of continuity}~\cite{devore1993constructive,dzyadyk2008theory,timan2014theory}.

\begin{definition}[Modulus of Continuity]
Let $f: \Omega \to \mathbb{R}$ be a continuous function on a domain $\Omega \subset \mathbb{R}^d$. The modulus of continuity $\omega_f(\delta)$ is defined as the maximum fluctuation of $f$ over any interval of size $\delta$:
\begin{equation}
    \omega_f(\delta) := \sup_{\substack{x, y \in \Omega \\ \|x - y\| \le \delta}} |f(x) - f(y)|.
\end{equation}
\end{definition}
Intuitively, $\omega_f(\delta)$ measures how ``wiggly'' the function is. For Lipschitz continuous functions with constant $K$, then the modulus of continuity is bounded as $\omega_f(\delta) \le K\delta$. The rate at which $\omega_f(\delta) \to 0$ as $\delta \to 0$ determines how quickly an estimator can converge to $f$.

\subsection{Proof of Effective Degree of DBNs Lemma}
\label{app:proof_lemma_degree}
\begin{proof}
We proceed by induction on the depth $L$.
\textit{Base Case ($L=1$):} The first layer computes $\sigma_n(\mathbf{W}^Tx + b)$. Since $\sigma_n$ is a polynomial of degree $n$ and the argument is linear in $x$, the output is a polynomial of degree $n = n^1$.

\textit{Inductive Step:} Assume the output of layer $l-1$, denoted $h^{(l-1)}(x)$, consists of polynomials of total degree $n^{l-1}$. The $l$-th layer computes:
\begin{equation}
    h^{(l)}(x) = \sum_{k=0}^n c_k B_{n,k}\left( \mathcal{L}(h^{(l-1)}(x)) \right),
\end{equation}
where $\mathcal{L}$ is a linear map and $B_{n,k}$ denotes the Bernstein basis $b_{n,k}^{[l,u]}(.)$. The basis function $B_{n,k}(\cdot)$ raises its input to the power $n$. Thus, the degree of the composition is the product of the degrees:
\begin{equation}
    \deg(h^{(l)}) = n \cdot \deg(h^{(l-1)}) = n \cdot n^{l-1} = n^l.
\end{equation}
By induction, the final output $\mathcal{N}(x)$ is a polynomial of degree $n^L$.
\end{proof}

\subsection{Proof of Theorem Exponential Approximation Rate}
\label{app:thrm_approximation_proof}
\begin{proof}
By Lemma \ref{lemma:degree}, the network output $\mathcal{N}(x)$ resides in $\Pi_{n^L}$, the space of polynomials of total degree $n^L$.
According to Jackson's Theorem for multivariate polynomial approximation \cite{devore1993constructive}, for any continuous function $f$, the error of the best polynomial approximation $P^* \in \Pi_{D}$ is bounded by:\vspace{-2mm}
\begin{equation}
    \inf_{P \in \Pi_D} \| P - f \|_\infty \le C_d \cdot \omega_f\left(\frac{1}{D}\right).
\end{equation}
Setting $D = n^L$, let $P^*_{n^L}$ be the best approximating polynomial in the full space $\Pi_{n^L}$. Its error satisfies:\vspace{-2mm}
\begin{equation}
    \| P^*_{n^L} - f \|_\infty \le C_d \cdot \omega_f\left(\frac{1}{n^L}\right).
\end{equation}
By Assumption \ref{ass:capacity}, the specific polynomial $P^*_{n^L}$ lies within the realizable manifold of our neural network $\mathcal{N}$. Thus, the network's optimal error coincides with the polynomial error:
\vspace{-2mm}
\begin{equation}
    \min_{\theta} \| \mathcal{N}(\cdot; \theta) - f \|_\infty \le C_d \cdot \omega\left(f, \frac{1}{n^L}\right).
    \vspace{-3mm}
\end{equation}
\end{proof}

\section{Experimental Setup}
\label{sec:exp_details}
All models were trained using the AdamW optimizer with an initial learning 
rate of $1 \times 10^{-3}$ and weight decay $1 \times 10^{-4}$. We used a 
batch size of 16,384. The learning rate was scheduled using 
\texttt{ReduceLROnPlateau} on validation AUC, with a reduction factor of 
0.5 and patience of 10 epochs, activated after epoch 5. Early stopping was 
applied with a patience of 15 epochs and a minimum AUC improvement threshold 
of $1 \times 10^{-3}$. Batch normalization was applied before each activation.

For all datasets, the data were split into 80\% train and 20\% test using a 
stratified split across multiple random seeds. Features were min-max 
normalized to $[0, 1]$ using training-set statistics only, with test 
features clipped to the same range.

All experiments were run on a single NVIDIA A30 GPU (24\,GB) with 4 CPU 
cores and 64\,GB of system memory.

\section{Additional Representational Capacity on Analytic Tasks: Regression Tasks (Experiment 1)}
\label{sec:regression}
Figure~\ref{fig:row_oscillator} shows that a single DBN neuron ($W{=}1$, $n{=}7$) fits the damped oscillator with an MSE of $2.59{\times}10^{-4}$, while a ReLU network with four times the width ($W{=}4$) achieves an MSE of $4.45{\times}10^{-3}$---a $17{\times}$ gap. The residual plot reveals that the ReLU network systematically fails to capture the oscillatory tail, whereas the DBN activation naturally adapts into a damped sinusoidal shape that mirrors the target's structure.
Similarly, on the Gaussian~Sine task (Figure~\ref{fig:row_gauss_sine}), a DBN with $W{=}2$ and $n{=}7$ achieves an MSE of $4.14{\times}10^{-3}$, outperforming the ReLU baseline ($W{=}4$, MSE~$= 5.58{\times}10^{-2}$) by over an order of magnitude.

\begin{figure}[t]
    \centering

    \begin{subfigure}[b]{\textwidth}
        \centering
        \includegraphics[width=0.32\textwidth]{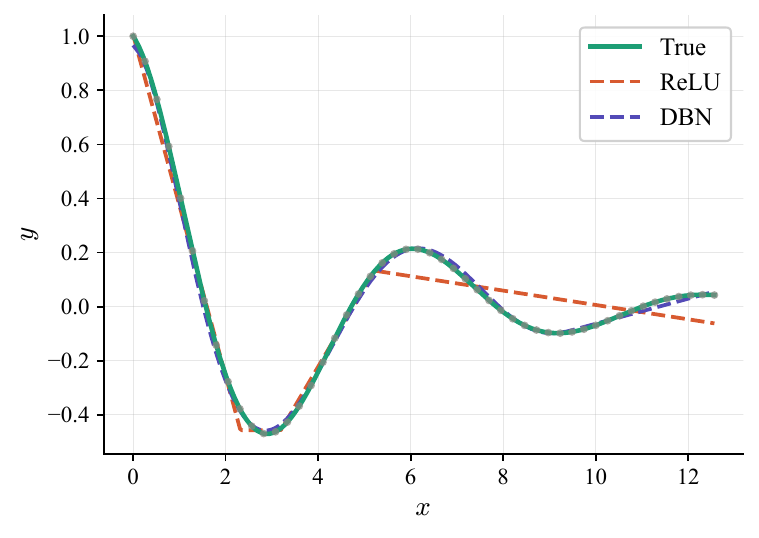}
        \hfill
        \includegraphics[width=0.32\textwidth]{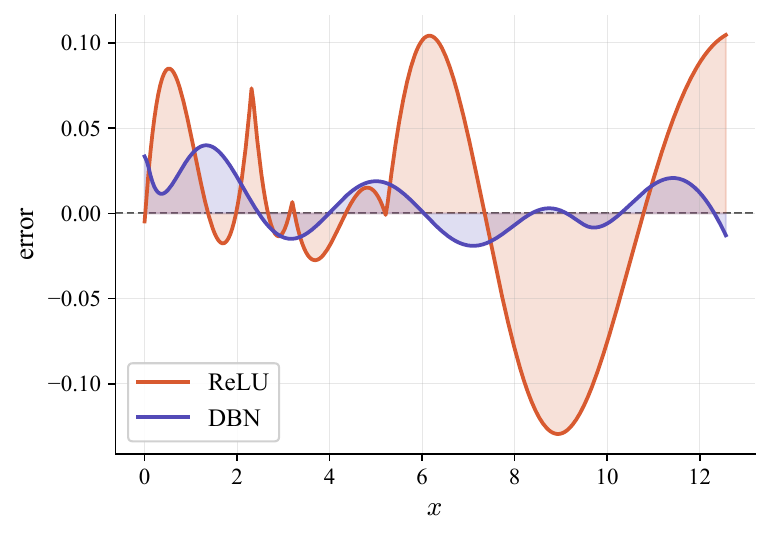}
        \hfill
        \includegraphics[width=0.32\textwidth]{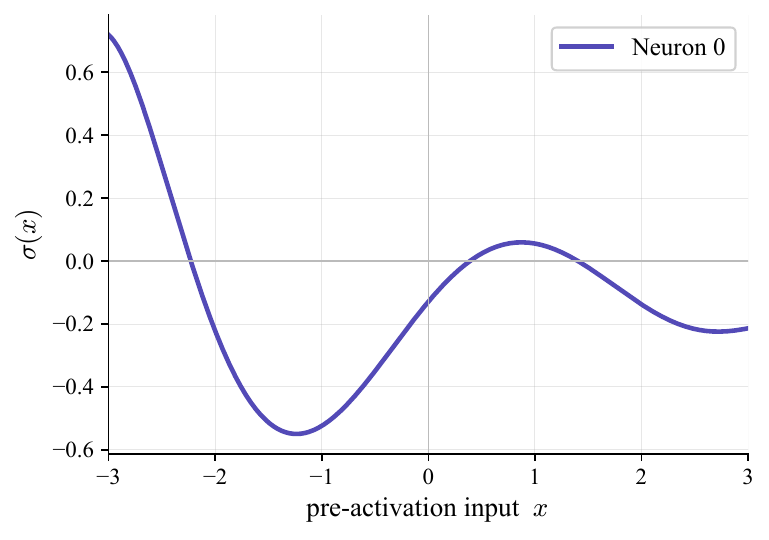}
        \caption{\textbf{Damped Oscillator} $y(x)=e^{-0.25x}\cos(x)$ --- DBN ($W{=}1$, $n{=}7$, MSE~$=2.59{\times}10^{-4}$) closely tracks the target while ReLU ($W{=}4$, MSE~$=4.45{\times}10^{-3}$) fails to capture the oscillatory tail. The learned activation adapts into a damped sinusoidal shape that mirrors the target's structure.}
        \label{fig:row_oscillator}
    \end{subfigure}
    \\[4pt]

    \begin{subfigure}[b]{\textwidth}
        \centering
        \includegraphics[width=0.32\textwidth]{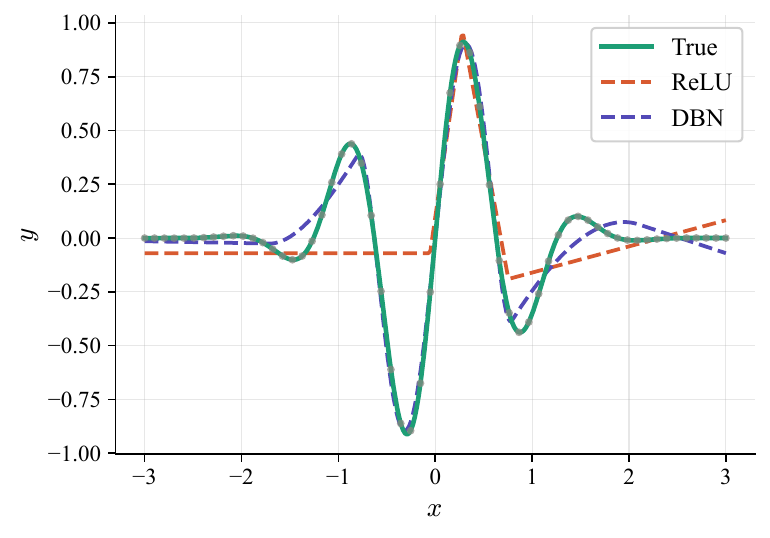}
        \hfill
        \includegraphics[width=0.32\textwidth]{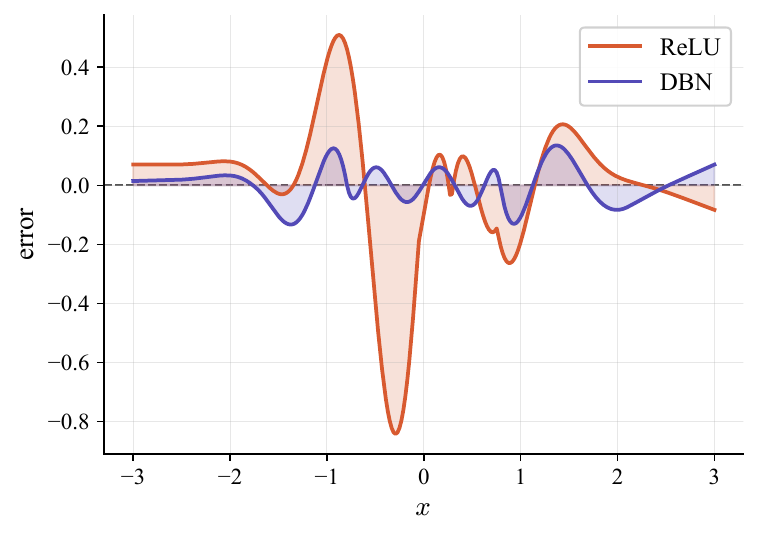}
        \hfill
        \includegraphics[width=0.32\textwidth]{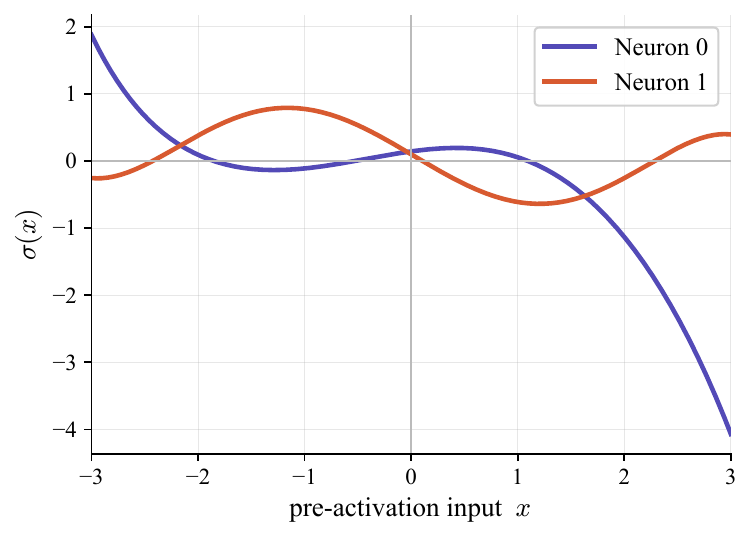}
        \caption{\textbf{Gaussian Sine} $y(x)=e^{-x^2}\sin(5x)$ --- DBN ($W{=}2$, $n{=}7$, MSE~$=4.14{\times}10^{-3}$) outperforms ReLU ($W{=}4$, MSE~$=5.58{\times}10^{-2}$) by over an order of magnitude; the ReLU network largely collapses to a near-constant prediction. The two learned activations develop complementary oscillatory shapes to jointly reconstruct the target's modulated waveform.}
        \label{fig:row_gauss_sine}
    \end{subfigure}

    \caption{\textbf{Regression tasks: function fit, pointwise residual, and learned Bernstein activations.}
    Each row shows the DBN and ReLU function fit (left), pointwise residual (center), and the learned Bernstein activation $\sigma(x)$ of the DBN hidden neurons (right).
    In both tasks the DBN achieves substantially lower error with fewer neurons, and the learned activations reveal how the network adapts its nonlinearity to match the target's structure.}
    \label{fig:regression}
\end{figure}

\section{Additional Parameter Efficiency Results (Experiment 2)}
\label{app:bucket_tables}

This appendix presents the complete parameter efficiency results for the six dataset--metric--residual combinations not included in the main text. The methodology is identical to that described in \S\ref{sec:exp_compression}: for each ReLU architecture ($L{\times}W$), we select the best-performing ReLU across seeds, identify all DBN models with a strictly better metric, and report the most parameter-efficient among them.

\subsection{HIGGS without residual connections}

Tables~\ref{tab:higgs_nores_auc_bucket_report} and~\ref{tab:higgs_nores_loss_bucket_report} present the results for HIGGS without residual connections, limited to $L \in \{5, 15\}$. The same scaling pattern from the main text holds: negative efficiency at $5{\times}16$ ($-20.2\%$) transitioning to positive values at deeper or wider configurations. At $15{\times}256$, efficiency reaches $92.2\%$ for AUC and $68.6\%$ for loss. The smaller depth range means the maximum efficiencies are lower than in the residual setting, but the monotonic increase with both $L$ and $W$ is clearly visible. For loss (Table~\ref{tab:higgs_nores_loss_bucket_report}), the best qualifying DBN at $15{\times}256$ achieves a loss of $0.3876$, compared to the ReLU baseline's $0.3955$---a meaningful improvement despite using $68.6\%$ fewer parameters.

\begin{table*}[t]
\centering
\caption{\textbf{Parameter efficiency on HIGGS (no residual connections, validation AUC).} Limited to $L \in \{5, 15\}$ due to vanishing gradients in deeper plain networks. Efficiency ranges from $-20.2\%$ at $5{\times}16$ to $92.2\%$ at $15{\times}256$. See Table~\ref{tab:higgs_residual_auc_bucket_report} for column definitions.}
\label{tab:higgs_nores_auc_bucket_report}
\resizebox{\textwidth}{!}{%
\setlength{\tabcolsep}{5pt}
\begin{tabular}{l r r r
    >{\centering\arraybackslash}p{2.2em}
    >{\centering\arraybackslash}p{2.2em}
    >{\centering\arraybackslash}p{2.2em}
    >{\centering\arraybackslash}p{2.2em} l r r}
\toprule
\multirow{2}{*}{\shortstack{Arch\\($L\times W$)}}  & \multicolumn{2}{c}{ReLU baseline}  & \multirow{2}{*}{\shortstack{Total\\qual.}}  & \multicolumn{4}{c}{Qualifying count by DBN degree}  & \multicolumn{2}{c}{Best DBN}  & \multirow{2}{*}{\shortstack{Param\\effic.}} \\
\cmidrule(lr){2-3} \cmidrule(lr){4-7} \cmidrule(lr){8-9}
  & \# params & auc  & & $n{=}3$ & $n{=}9$ & $n{=}15$ & $n{=}25$  & \# params (arch.,~$n$) & auc  & \\
\midrule
  $5\times16$ & 1,586 & 0.8212 & 96 & 24 & 24 & 24 & 24 & 1,906~~($5\times16$,~$n{=}3$) & 0.8300 & $-20.2\%$ \\
  $5\times32$ & 5,218 & 0.8403 & 72 & 18 & 18 & 18 & 18 & 5,858~~($5\times32$,~$n{=}3$) & 0.8486 & $-12.3\%$ \\
  $5\times128$ & 70,018 & 0.8654 & 48 & 12 & 12 & 12 & 12 & 72,578~~($5\times128$,~$n{=}3$) & 0.8783 & $-3.7\%$ \\
  $5\times256$ & 271,106 & 0.8711 & 43 & 7 & 12 & 12 & 12 & 72,578~~($5\times128$,~$n{=}3$) & 0.8783 & $73.2\%$ \\
\midrule
  $15\times16$ & 4,306 & 0.8148 & 96 & 24 & 24 & 24 & 24 & 1,906~~($5\times16$,~$n{=}3$) & 0.8300 & $55.7\%$ \\
  $15\times32$ & 15,778 & 0.8393 & 73 & 19 & 18 & 18 & 18 & 5,266~~($15\times16$,~$n{=}3$) & 0.8402 & $66.6\%$ \\
  $15\times128$ & 235,138 & 0.8675 & 48 & 12 & 12 & 12 & 12 & 72,578~~($5\times128$,~$n{=}3$) & 0.8783 & $69.1\%$ \\
  $15\times256$ & 929,026 & 0.8709 & 44 & 8 & 12 & 12 & 12 & 72,578~~($5\times128$,~$n{=}3$) & 0.8783 & $92.2\%$ \\
\bottomrule
\end{tabular}%
}
\end{table*}

\begin{table*}[t]
\centering
\caption{\textbf{Parameter efficiency on HIGGS (no residual connections, training loss).} Efficiency ranges from $-20.2\%$ at $5{\times}16$ to $68.6\%$ at $15{\times}256$. At $5{\times}32$, the efficiency is near zero ($-0.9\%$), indicating the crossover point. See Table~\ref{tab:higgs_residual_auc_bucket_report} for column definitions.}
\label{tab:higgs_nores_loss_bucket_report}
\resizebox{\textwidth}{!}{%
\setlength{\tabcolsep}{5pt}
\begin{tabular}{l r r r
    >{\centering\arraybackslash}p{2.2em}
    >{\centering\arraybackslash}p{2.2em}
    >{\centering\arraybackslash}p{2.2em}
    >{\centering\arraybackslash}p{2.2em} l r r}
\toprule
\multirow{2}{*}{\shortstack{Arch\\($L\times W$)}}  & \multicolumn{2}{c}{ReLU baseline}  & \multirow{2}{*}{\shortstack{Total\\qual.}}  & \multicolumn{4}{c}{Qualifying count by DBN degree}  & \multicolumn{2}{c}{Best DBN}  & \multirow{2}{*}{\shortstack{Param\\effic.}} \\
\cmidrule(lr){2-3} \cmidrule(lr){4-7} \cmidrule(lr){8-9}
  & \# params & loss  & & $n{=}3$ & $n{=}9$ & $n{=}15$ & $n{=}25$  & \# params (arch.,~$n$) & loss  & \\
\midrule
  $5\times16$ & 1,586 & 0.5132 & 96 & 24 & 24 & 24 & 24 & 1,906~~($5\times16$,~$n{=}3$) & 0.5018 & $-20.2\%$ \\
  $5\times32$ & 5,218 & 0.4887 & 73 & 19 & 18 & 18 & 18 & 5,266~~($15\times16$,~$n{=}3$) & 0.4885 & $-0.9\%$ \\
  $5\times128$ & 70,018 & 0.4443 & 48 & 12 & 12 & 12 & 12 & 72,578~~($5\times128$,~$n{=}3$) & 0.4243 & $-3.7\%$ \\
  $5\times256$ & 271,106 & 0.4105 & 26 & 5 & 7 & 6 & 8 & 254,338~~($15\times128$,~$n{=}9$) & 0.4048 & $6.2\%$ \\
\midrule
  $15\times16$ & 4,306 & 0.5207 & 96 & 24 & 24 & 24 & 24 & 1,906~~($5\times16$,~$n{=}3$) & 0.5018 & $55.7\%$ \\
  $15\times32$ & 15,778 & 0.4887 & 73 & 19 & 18 & 18 & 18 & 5,266~~($15\times16$,~$n{=}3$) & 0.4885 & $66.6\%$ \\
  $15\times128$ & 235,138 & 0.4345 & 48 & 12 & 12 & 12 & 12 & 72,578~~($5\times128$,~$n{=}3$) & 0.4243 & $69.1\%$ \\
  $15\times256$ & 929,026 & 0.3955 & 14 & 3 & 2 & 4 & 5 & 291,586~~($5\times256$,~$n{=}15$) & 0.3876 & $68.6\%$ \\
\bottomrule
\end{tabular}%
}
\end{table*}

\subsection{HIGGS with residual connections --- training loss}

Table~\ref{tab:higgs_residual_loss_bucket_report} presents the loss-based analysis for HIGGS with residual connections. The pattern mirrors the AUC results from the main text but with an important nuance: at the widest architectures within each depth group, the efficiency for loss is somewhat lower than for AUC. For instance, at $15{\times}256$ the efficiency is only $5.6\%$ for loss versus $92.2\%$ for AUC, because the best ReLU at this architecture achieves a strong loss of $0.3862$ that is harder to beat with a smaller model. Nevertheless, at $L \geq 50$, efficiencies consistently exceed $67\%$, and at $L{=}150$ they range from $89\%$ to $95\%$ across all widths.

\begin{table*}[t]
\centering
\caption{\textbf{Parameter efficiency on HIGGS (residual connections, training loss).} Efficiencies range from $-20.2\%$ at $5{\times}16$ to $95.4\%$ at $150{\times}16$. The loss metric shows wider variation than AUC, with some large-width architectures showing modest efficiency at intermediate depths. See Table~\ref{tab:higgs_residual_auc_bucket_report} for column definitions.}
\label{tab:higgs_residual_loss_bucket_report}
\resizebox{\textwidth}{!}{%
\setlength{\tabcolsep}{5pt}
\begin{tabular}{l r r r
    >{\centering\arraybackslash}p{2.2em}
    >{\centering\arraybackslash}p{2.2em}
    >{\centering\arraybackslash}p{2.2em}
    >{\centering\arraybackslash}p{2.2em} l r r}
\toprule
\multirow{2}{*}{\shortstack{Arch\\($L\times W$)}}  & \multicolumn{2}{c}{ReLU baseline}  & \multirow{2}{*}{\shortstack{Total\\qual.}}  & \multicolumn{4}{c}{Qualifying count by DBN degree}  & \multicolumn{2}{c}{Best DBN}  & \multirow{2}{*}{\shortstack{Param\\effic.}} \\
\cmidrule(lr){2-3} \cmidrule(lr){4-7} \cmidrule(lr){8-9}
  & \# params & loss  & & $n{=}3$ & $n{=}9$ & $n{=}15$ & $n{=}25$  & \# params (arch.,~$n$) & loss  & \\
\midrule
  $5\times16$ & 1,586 & 0.5135 & 240 & 60 & 60 & 60 & 60 & 1,906~~($5\times16$,~$n{=}3$) & 0.5032 & $-20.2\%$ \\
  $5\times32$ & 5,218 & 0.4896 & 222 & 54 & 57 & 57 & 54 & 5,858~~($5\times32$,~$n{=}3$) & 0.4814 & $-12.3\%$ \\
  $5\times128$ & 70,018 & 0.4479 & 133 & 32 & 34 & 34 & 33 & 72,578~~($5\times128$,~$n{=}3$) & 0.4337 & $-3.7\%$ \\
  $5\times256$ & 271,106 & 0.4182 & 106 & 23 & 26 & 28 & 29 & 80,258~~($5\times128$,~$n{=}15$) & 0.4113 & $70.4\%$ \\
\midrule
  $15\times16$ & 4,306 & 0.5120 & 240 & 60 & 60 & 60 & 60 & 1,906~~($5\times16$,~$n{=}3$) & 0.5032 & $55.7\%$ \\
  $15\times32$ & 15,778 & 0.4809 & 202 & 49 & 53 & 53 & 47 & 6,818~~($5\times32$,~$n{=}9$) & 0.4747 & $56.8\%$ \\
  $15\times128$ & 235,138 & 0.4220 & 114 & 26 & 28 & 30 & 30 & 76,418~~($5\times128$,~$n{=}9$) & 0.4203 & $67.5\%$ \\
  $15\times256$ & 929,026 & 0.3862 & 75 & 13 & 20 & 21 & 21 & 877,058~~($50\times128$,~$n{=}9$) & 0.3707 & $5.6\%$ \\
\midrule
  $50\times16$ & 13,826 & 0.5059 & 238 & 59 & 60 & 60 & 59 & 1,906~~($5\times16$,~$n{=}3$) & 0.5032 & $86.2\%$ \\
  $50\times32$ & 52,738 & 0.4785 & 192 & 46 & 51 & 50 & 45 & 6,818~~($5\times32$,~$n{=}9$) & 0.4747 & $87.1\%$ \\
  $50\times128$ & 813,058 & 0.4098 & 94 & 20 & 22 & 25 & 27 & 265,858~~($15\times128$,~$n{=}15$) & 0.4055 & $67.3\%$ \\
  $50\times256$ & 3,231,746 & 0.3696 & 62 & 9 & 14 & 20 & 19 & 915,458~~($50\times128$,~$n{=}15$) & 0.3590 & $71.7\%$ \\
\midrule
  $100\times16$ & 27,426 & 0.5058 & 238 & 59 & 60 & 60 & 59 & 1,906~~($5\times16$,~$n{=}3$) & 0.5032 & $93.1\%$ \\
  $100\times32$ & 105,538 & 0.4720 & 170 & 42 & 42 & 43 & 43 & 9,378~~($5\times32$,~$n{=}25$) & 0.4704 & $91.1\%$ \\
  $100\times128$ & 1,638,658 & 0.3928 & 79 & 14 & 21 & 21 & 23 & 304,386~~($5\times256$,~$n{=}25$) & 0.3886 & $81.4\%$ \\
  $100\times256$ & 6,521,346 & 0.3620 & 57 & 8 & 11 & 19 & 19 & 915,458~~($50\times128$,~$n{=}15$) & 0.3590 & $86.0\%$ \\
\midrule
  $150\times16$ & 41,026 & 0.5055 & 238 & 59 & 60 & 60 & 59 & 1,906~~($5\times16$,~$n{=}3$) & 0.5032 & $95.4\%$ \\
  $150\times32$ & 158,338 & 0.4719 & 170 & 42 & 42 & 43 & 43 & 9,378~~($5\times32$,~$n{=}25$) & 0.4704 & $94.1\%$ \\
  $150\times128$ & 2,464,258 & 0.4098 & 94 & 20 & 22 & 25 & 27 & 265,858~~($15\times128$,~$n{=}15$) & 0.4055 & $89.2\%$ \\
  $150\times256$ & 9,810,946 & 0.3690 & 62 & 9 & 14 & 20 & 19 & 915,458~~($50\times128$,~$n{=}15$) & 0.3590 & $90.7\%$ \\
\bottomrule
\end{tabular}%
}
\end{table*}

\subsection{SUSY without residual connections}

Tables~\ref{tab:susy_nores_auc_bucket_report} and~\ref{tab:susy_nores_loss_bucket_report} present results for SUSY without residual connections. Consistent with the main-text findings, the efficiencies on SUSY are higher than on HIGGS at comparable architectures. For AUC, the efficiency reaches $99.7\%$ at $15{\times}256$, with a $2{,}706$-parameter DBN ($5{\times}16$, $n{=}15$) matching a $926$K-parameter ReLU. Even at $5{\times}128$, the efficiency is already $96.1\%$. For loss (Table~\ref{tab:susy_nores_loss_bucket_report}), the pattern is slightly more variable: the efficiency at $5{\times}128$ is $-9.3\%$ (the best qualifying DBN requires more parameters), but it rises to $91.9\%$ at $15{\times}256$. Notably, the loss table shows that $97.6\%$ efficiency is achievable on SUSY at $15{\times}128$ for the loss metric (from the residual setting in Table~\ref{tab:susy_residual_loss_bucket_report}), but here without residuals the best qualifying DBN at $15{\times}128$ achieves $99.0\%$ for AUC, illustrating that the two metrics can yield somewhat different efficiency profiles.

\begin{table*}[t]
\centering
\caption{\textbf{Parameter efficiency on SUSY (no residual connections, validation AUC).} Limited to $L \in \{5, 15\}$. Efficiencies reach $99.7\%$ at $15{\times}256$, with only the $5{\times}16$ architecture showing negative efficiency ($-22.4\%$). See Table~\ref{tab:higgs_residual_auc_bucket_report} for column definitions.}
\label{tab:susy_nores_auc_bucket_report}
\resizebox{\textwidth}{!}{%
\setlength{\tabcolsep}{5pt}
\begin{tabular}{l r r r
    >{\centering\arraybackslash}p{2.2em}
    >{\centering\arraybackslash}p{2.2em}
    >{\centering\arraybackslash}p{2.2em}
    >{\centering\arraybackslash}p{2.2em} l r r}
\toprule
\multirow{2}{*}{\shortstack{Arch\\($L\times W$)}}  & \multicolumn{2}{c}{ReLU baseline}  & \multirow{2}{*}{\shortstack{Total\\qual.}}  & \multicolumn{4}{c}{Qualifying count by DBN degree}  & \multicolumn{2}{c}{Best DBN}  & \multirow{2}{*}{\shortstack{Param\\effic.}} \\
\cmidrule(lr){2-3} \cmidrule(lr){4-7} \cmidrule(lr){8-9}
  & \# params & auc  & & $n{=}3$ & $n{=}9$ & $n{=}15$ & $n{=}25$  & \# params (arch.,~$n$) & auc  & \\
\midrule
  $5\times16$ & 1,426 & 0.8759 & 93 & 21 & 24 & 24 & 24 & 1,746~~($5\times16$,~$n{=}3$) & 0.8767 & $-22.4\%$ \\
  $5\times32$ & 4,898 & 0.8767 & 81 & 18 & 22 & 22 & 19 & 1,746~~($5\times16$,~$n{=}3$) & 0.8767 & $64.4\%$ \\
  $5\times128$ & 68,738 & 0.8773 & 50 & 15 & 14 & 15 & 6 & 2,706~~($5\times16$,~$n{=}15$) & 0.8775 & $96.1\%$ \\
  $5\times256$ & 268,546 & 0.8775 & 40 & 12 & 12 & 12 & 4 & 2,706~~($5\times16$,~$n{=}15$) & 0.8775 & $99.0\%$ \\
\midrule
  $15\times16$ & 4,146 & 0.8763 & 91 & 19 & 24 & 24 & 24 & 1,746~~($5\times16$,~$n{=}3$) & 0.8767 & $57.9\%$ \\
  $15\times32$ & 15,458 & 0.8773 & 50 & 15 & 14 & 15 & 6 & 2,706~~($5\times16$,~$n{=}15$) & 0.8775 & $82.5\%$ \\
  $15\times128$ & 233,858 & 0.8768 & 73 & 17 & 22 & 19 & 15 & 2,226~~($5\times16$,~$n{=}9$) & 0.8769 & $99.0\%$ \\
  $15\times256$ & 926,466 & 0.8772 & 55 & 15 & 16 & 17 & 7 & 2,706~~($5\times16$,~$n{=}15$) & 0.8775 & $99.7\%$ \\
\bottomrule
\end{tabular}%
}
\end{table*}

\begin{table*}[t]
\centering
\caption{\textbf{Parameter efficiency on SUSY (no residual connections, training loss).} The $5{\times}128$ architecture shows negative efficiency ($-9.3\%$), indicating that at this scale the best qualifying DBN is slightly larger than the ReLU. Efficiency recovers to $91.9\%$ at $15{\times}256$. See Table~\ref{tab:higgs_residual_auc_bucket_report} for column definitions.}
\label{tab:susy_nores_loss_bucket_report}
\resizebox{\textwidth}{!}{%
\setlength{\tabcolsep}{5pt}
\begin{tabular}{l r r r
    >{\centering\arraybackslash}p{2.2em}
    >{\centering\arraybackslash}p{2.2em}
    >{\centering\arraybackslash}p{2.2em}
    >{\centering\arraybackslash}p{2.2em} l r r}
\toprule
\multirow{2}{*}{\shortstack{Arch\\($L\times W$)}}  & \multicolumn{2}{c}{ReLU baseline}  & \multirow{2}{*}{\shortstack{Total\\qual.}}  & \multicolumn{4}{c}{Qualifying count by DBN degree}  & \multicolumn{2}{c}{Best DBN}  & \multirow{2}{*}{\shortstack{Param\\effic.}} \\
\cmidrule(lr){2-3} \cmidrule(lr){4-7} \cmidrule(lr){8-9}
  & \# params & loss  & & $n{=}3$ & $n{=}9$ & $n{=}15$ & $n{=}25$  & \# params (arch.,~$n$) & loss  & \\
\midrule
  $5\times16$ & 1,426 & 0.4262 & 93 & 22 & 24 & 24 & 23 & 1,746~~($5\times16$,~$n{=}3$) & 0.4251 & $-22.4\%$ \\
  $5\times32$ & 4,898 & 0.4246 & 78 & 18 & 19 & 21 & 20 & 3,506~~($5\times16$,~$n{=}25$) & 0.4244 & $28.4\%$ \\
  $5\times128$ & 68,738 & 0.4215 & 36 & 4 & 10 & 10 & 12 & 75,138~~($5\times128$,~$n{=}9$) & 0.4195 & $-9.3\%$ \\
  $5\times256$ & 268,546 & 0.4187 & 17 & 0 & 3 & 5 & 9 & 85,378~~($5\times128$,~$n{=}25$) & 0.4173 & $68.2\%$ \\
\midrule
  $15\times16$ & 4,146 & 0.4259 & 91 & 21 & 24 & 24 & 22 & 1,746~~($5\times16$,~$n{=}3$) & 0.4251 & $57.9\%$ \\
  $15\times32$ & 15,458 & 0.4241 & 69 & 17 & 18 & 19 & 15 & 5,538~~($5\times32$,~$n{=}3$) & 0.4239 & $64.2\%$ \\
  $15\times128$ & 233,858 & 0.4191 & 19 & 0 & 3 & 6 & 10 & 85,378~~($5\times128$,~$n{=}25$) & 0.4173 & $63.5\%$ \\
  $15\times256$ & 926,466 & 0.4213 & 36 & 4 & 10 & 10 & 12 & 75,138~~($5\times128$,~$n{=}9$) & 0.4195 & $91.9\%$ \\
\bottomrule
\end{tabular}%
}
\end{table*}

\subsection{SUSY with residual connections --- training loss}

Table~\ref{tab:susy_residual_loss_bucket_report} completes the picture for SUSY with the training loss metric. The results are highly consistent with the AUC tables in the main text: efficiency reaches $\mathbf{100.0\%}$ at $150{\times}256$, where a DBN with just $3{,}506$ parameters ($5{\times}16$, $n{=}25$) achieves a lower training loss than a $9.8$M-parameter ReLU. The loss metric shows one notable difference from AUC: at intermediate architectures with large width (e.g., $5{\times}128$), the efficiency can be negative ($-9.3\%$), because the ReLU's loss is already competitive and the smallest qualifying DBN happens to be slightly larger. However, this effect disappears rapidly with depth: by $L{=}50$, all efficiencies exceed $67\%$, and by $L{=}150$, all exceed $94\%$.

\begin{table*}[t]
\centering
\caption{\textbf{Parameter efficiency on SUSY (residual connections, training loss).} Efficiency reaches $\mathbf{100.0\%}$ at $150{\times}256$, matching the AUC result. At $L \geq 100$, efficiencies exceed $83\%$ for all widths. See Table~\ref{tab:higgs_residual_auc_bucket_report} for column definitions.}
\label{tab:susy_residual_loss_bucket_report}
\resizebox{\textwidth}{!}{%
\setlength{\tabcolsep}{5pt}
\begin{tabular}{l r r r
    >{\centering\arraybackslash}p{2.2em}
    >{\centering\arraybackslash}p{2.2em}
    >{\centering\arraybackslash}p{2.2em}
    >{\centering\arraybackslash}p{2.2em} l r r}
\toprule
\multirow{2}{*}{\shortstack{Arch\\($L\times W$)}}  & \multicolumn{2}{c}{ReLU baseline}  & \multirow{2}{*}{\shortstack{Total\\qual.}}  & \multicolumn{4}{c}{Qualifying count by DBN degree}  & \multicolumn{2}{c}{Best DBN}  & \multirow{2}{*}{\shortstack{Param\\effic.}} \\
\cmidrule(lr){2-3} \cmidrule(lr){4-7} \cmidrule(lr){8-9}
  & \# params & loss  & & $n{=}3$ & $n{=}9$ & $n{=}15$ & $n{=}25$  & \# params (arch.,~$n$) & loss  & \\
\midrule
  $5\times16$ & 1,426 & 0.4268 & 239 & 59 & 60 & 60 & 60 & 1,746~~($5\times16$,~$n{=}3$) & 0.4259 & $-22.4\%$ \\
  $5\times32$ & 4,898 & 0.4248 & 216 & 49 & 53 & 57 & 57 & 2,226~~($5\times16$,~$n{=}9$) & 0.4246 & $54.6\%$ \\
  $5\times128$ & 68,738 & 0.4221 & 114 & 20 & 30 & 33 & 31 & 75,138~~($5\times128$,~$n{=}9$) & 0.4205 & $-9.3\%$ \\
  $5\times256$ & 268,546 & 0.4222 & 121 & 25 & 30 & 34 & 32 & 75,138~~($5\times128$,~$n{=}9$) & 0.4205 & $72.0\%$ \\
\midrule
  $15\times16$ & 4,146 & 0.4264 & 239 & 59 & 60 & 60 & 60 & 1,746~~($5\times16$,~$n{=}3$) & 0.4259 & $57.9\%$ \\
  $15\times32$ & 15,458 & 0.4245 & 209 & 47 & 51 & 57 & 54 & 2,706~~($5\times16$,~$n{=}15$) & 0.4244 & $82.5\%$ \\
  $15\times128$ & 233,858 & 0.4214 & 97 & 13 & 26 & 27 & 31 & 75,138~~($5\times128$,~$n{=}9$) & 0.4205 & $67.9\%$ \\
  $15\times256$ & 926,466 & 0.4212 & 97 & 13 & 26 & 27 & 31 & 75,138~~($5\times128$,~$n{=}9$) & 0.4205 & $91.9\%$ \\
\midrule
  $50\times16$ & 13,666 & 0.4257 & 230 & 53 & 59 & 59 & 59 & 2,226~~($5\times16$,~$n{=}9$) & 0.4246 & $83.7\%$ \\
  $50\times32$ & 52,418 & 0.4253 & 224 & 50 & 57 & 58 & 59 & 2,226~~($5\times16$,~$n{=}9$) & 0.4246 & $95.8\%$ \\
  $50\times128$ & 811,778 & 0.4172 & 48 & 1 & 10 & 17 & 20 & 264,578~~($15\times128$,~$n{=}15$) & 0.4167 & $67.4\%$ \\
  $50\times256$ & 3,229,186 & 0.4178 & 55 & 2 & 11 & 19 & 23 & 264,578~~($15\times128$,~$n{=}15$) & 0.4167 & $91.8\%$ \\
\midrule
  $100\times16$ & 27,266 & 0.4262 & 237 & 57 & 60 & 60 & 60 & 1,746~~($5\times16$,~$n{=}3$) & 0.4259 & $93.6\%$ \\
  $100\times32$ & 105,218 & 0.4238 & 185 & 40 & 48 & 49 & 48 & 6,498~~($5\times32$,~$n{=}9$) & 0.4233 & $93.8\%$ \\
  $100\times128$ & 1,637,378 & 0.4190 & 69 & 5 & 14 & 23 & 27 & 264,578~~($15\times128$,~$n{=}15$) & 0.4167 & $83.8\%$ \\
  $100\times256$ & 6,518,786 & 0.4230 & 145 & 30 & 40 & 37 & 38 & 9,058~~($5\times32$,~$n{=}25$) & 0.4228 & $99.9\%$ \\
\midrule
  $150\times16$ & 40,866 & 0.4253 & 225 & 50 & 57 & 59 & 59 & 2,226~~($5\times16$,~$n{=}9$) & 0.4246 & $94.6\%$ \\
  $150\times32$ & 158,018 & 0.4254 & 226 & 51 & 57 & 59 & 59 & 2,226~~($5\times16$,~$n{=}9$) & 0.4246 & $98.6\%$ \\
  $150\times128$ & 2,462,978 & 0.4222 & 121 & 25 & 30 & 34 & 32 & 75,138~~($5\times128$,~$n{=}9$) & 0.4205 & $96.9\%$ \\
  $150\times256$ & 9,808,386 & 0.4243 & 203 & 44 & 51 & 55 & 53 & 3,506~~($5\times16$,~$n{=}25$) & 0.4240 & $100.0\%$ \\
\bottomrule
\end{tabular}%
}
\end{table*}


\section{Additional Convergence Results (Experiment 3)}
\label{app:convergence}

This appendix presents the convergence analysis for configurations without residual connections. The methodology is identical to \S\ref{sec:exp_convergence}: the crossover epoch is the first epoch at which the DBN's median training loss (over 3 seeds) falls below the ReLU's final median loss, and the loss improvement is measured at the end of training. Total epochs are the median across seeds.

\subsection{HIGGS without residual connections}

Table~\ref{tab:crossover_higgs_nores} shows the crossover analysis for HIGGS without residual connections. The same qualitative patterns from the residual setting hold: crossover epochs decrease with higher degree, and loss improvements increase with both depth and degree. The improvements are somewhat smaller than in the residual setting---the maximum is $7.5\%$ at $5{\times}128$ with $n{=}25$ and $5{\times}128$ with $n{=}15$---reflecting the shallower networks ($L \leq 15$) available without residual connections.

\begin{table}[t]
\centering
\setlength{\tabcolsep}{3pt}
\caption{\textbf{Convergence analysis on HIGGS (no residual connections).} Same methodology as Table~\ref{tab:crossover_higgs_res}. Limited to $L \in \{5, 15\}$. The crossover and scaling patterns are consistent with the residual setting, though with smaller absolute improvements due to the restricted depth range.}
\label{tab:crossover_higgs_nores}
\begin{tabular}{lccccc}
\toprule
\textbf{Arch} & \textbf{ReLU} & \multicolumn{4}{c}{\textbf{DBN}} \\
\cmidrule(lr){3-6}
($L{\times}W$) & (epochs) & $n=3$ & $n=9$ & $n=15$ & $n=25$ \\
\midrule
5$\times$16 & 66 & 18/77 ($\downarrow$2.2\%) & 14/83 ($\downarrow$3.4\%) & 13/69 ($\downarrow$3.3\%) & 17/82 ($\downarrow$2.5\%) \\
5$\times$32 & 69 & 26/91 ($\downarrow$2.3\%) & 15/79 ($\downarrow$3.1\%) & 13/94 ($\downarrow$3.4\%) & 15/90 ($\downarrow$3.5\%) \\
5$\times$128 & 75 & 34/135 ($\downarrow$4.7\%) & 17/108 ($\downarrow$6.4\%) & 15/99 ($\downarrow$7.3\%) & 11/84 ($\downarrow$7.5\%) \\
5$\times$256 & 69 & 59/60 ($\downarrow$0.2\%) & 31/45 ($\downarrow$2.1\%) & 23/50 ($\downarrow$5.1\%) & 19/37 ($\downarrow$5.2\%) \\
\midrule
15$\times$16 & 96 & 11/104 ($\downarrow$6.2\%) & 17/88 ($\downarrow$6.1\%) & 14/79 ($\downarrow$5.7\%) & 15/106 ($\downarrow$4.0\%) \\
15$\times$32 & 74 & 12/91 ($\downarrow$5.1\%) & 11/99 ($\downarrow$4.8\%) & 14/90 ($\downarrow$5.1\%) & 13/80 ($\downarrow$4.4\%) \\
15$\times$128 & 69 & 30/74 ($\downarrow$4.8\%) & 19/78 ($\downarrow$6.8\%) & 18/68 ($\downarrow$5.2\%) & 15/72 ($\downarrow$7.3\%) \\
15$\times$256 & 49 & 39/41 ($\downarrow$1.4\%) & 28/33 ($\downarrow$4.8\%) & 24/30 ($\downarrow$4.7\%) & 21/28 ($\downarrow$4.7\%) \\
\bottomrule
\end{tabular}
\end{table}

\subsection{SUSY without residual connections}

Table~\ref{tab:crossover_susy_nores} shows the results for SUSY without residual connections. The improvements are the smallest across all four settings, consistent with SUSY being the easier task and the non-residual setting limiting depth to $L \leq 15$. Several ``--'' entries appear for $n{=}3$ at larger widths ($5{\times}128$, $5{\times}256$, $15{\times}128$) and one for $n{=}15$ at $15{\times}128$ and $n{=}25$ at $15{\times}32$, indicating that at these architectures, the ReLU baseline is already well-optimized and only higher-degree DBNs can surpass it. The maximum improvement is $4.5\%$ at $15{\times}256$ with $n{=}25$. Despite the modest magnitudes, the degree-scaling pattern persists: higher $n$ consistently achieves earlier crossover and larger improvement.

\begin{table}[t]
\centering
\setlength{\tabcolsep}{3pt}
\caption{\textbf{Convergence analysis on SUSY (no residual connections).} Same methodology as Table~\ref{tab:crossover_higgs_res}. Limited to $L \in \{5, 15\}$. More ``--'' entries appear than in other settings, reflecting the combined effect of shallow depth and SUSY's lower intrinsic difficulty, which limits the margin by which DBNs can surpass ReLU.}
\label{tab:crossover_susy_nores}
\begin{tabular}{lccccc}
\toprule
\textbf{Arch} & \textbf{ReLU} & \multicolumn{4}{c}{\textbf{DBN}} \\
\cmidrule(lr){3-6}
($L{\times}W$) & (epochs) & $n=3$ & $n=9$ & $n=15$ & $n=25$ \\
\midrule
5$\times$16 & 44 & 26/39 ($\downarrow$0.3\%) & 22/39 ($\downarrow$0.3\%) & 19/37 ($\downarrow$0.4\%) & 19/37 ($\downarrow$0.4\%) \\
5$\times$32 & 42 & 26/40 ($\downarrow$0.2\%) & 19/39 ($\downarrow$0.3\%) & 20/34 ($\downarrow$0.3\%) & 22/37 ($\downarrow$0.4\%) \\
5$\times$128 & 40 & --/33 ($\uparrow$0.2\%) & 32/38 ($\downarrow$0.5\%) & 25/26 ($\downarrow$0.2\%) & 22/32 ($\downarrow$1.0\%) \\
5$\times$256 & 39 & --/37 ($\uparrow$0.3\%) & 26/27 ($\downarrow$0.2\%) & 23/29 ($\downarrow$1.9\%) & 19/27 ($\downarrow$3.1\%) \\
\midrule
15$\times$16 & 58 & 38/44 ($\downarrow$0.2\%) & 53/83 ($\downarrow$0.5\%) & 43/70 ($\downarrow$0.4\%) & 34/46 ($\downarrow$0.3\%) \\
15$\times$32 & 54 & 33/48 ($\downarrow$0.1\%) & 27/44 ($\downarrow$0.2\%) & 29/36 ($\downarrow$0.2\%) & --/34 ($\uparrow$0.0\%) \\
15$\times$128 & 41 & --/30 ($\uparrow$0.3\%) & 27/27 ($\downarrow$0.1\%) & --/24 ($\uparrow$0.2\%) & 25/28 ($\downarrow$0.9\%) \\
15$\times$256 & 30 & 30/33 ($\downarrow$0.1\%) & 19/24 ($\downarrow$0.7\%) & 17/25 ($\downarrow$2.0\%) & 14/25 ($\downarrow$4.5\%) \\
\bottomrule
\end{tabular}
\end{table}


\section{Additional Distributional Analysis (Experiment 4)}
\label{app:distributional}

Figures~\ref{fig:auc_histograms_nores} and~\ref{fig:loss_histograms_nores} present the stacked histograms for configurations without residual connections. The same qualitative separation observed in the residual setting (\S\ref{sec:exp_distributional}) holds: DBN variants dominate the high-AUC and low-loss tails, while all four standard activations (ReLU, Leaky ReLU, SELU, GeLU) cluster together. On HIGGS without residual connections (Figure~\ref{fig:higgs_nores_auc_hist_app}), DBNs exclusively occupy all AUC buckets above $0.875$, reaching up to $0.887$. On SUSY (Figure~\ref{fig:susy_nores_auc_hist_app}), the separation is visible but narrower, consistent with both methods approaching SUSY's performance ceiling in the shallow-depth regime ($L \leq 15$). For training loss on HIGGS (Figure~\ref{fig:higgs_nores_loss_hist_app}), the entire region below $0.40$ is DBN-exclusive, while on SUSY (Figure~\ref{fig:susy_nores_loss_hist_app}) losses below $0.415$ are dominated by higher-degree DBNs.

\begin{figure}[t]
    \centering
    \begin{subfigure}[b]{0.48\textwidth}
        \centering
        \includegraphics[width=\textwidth]{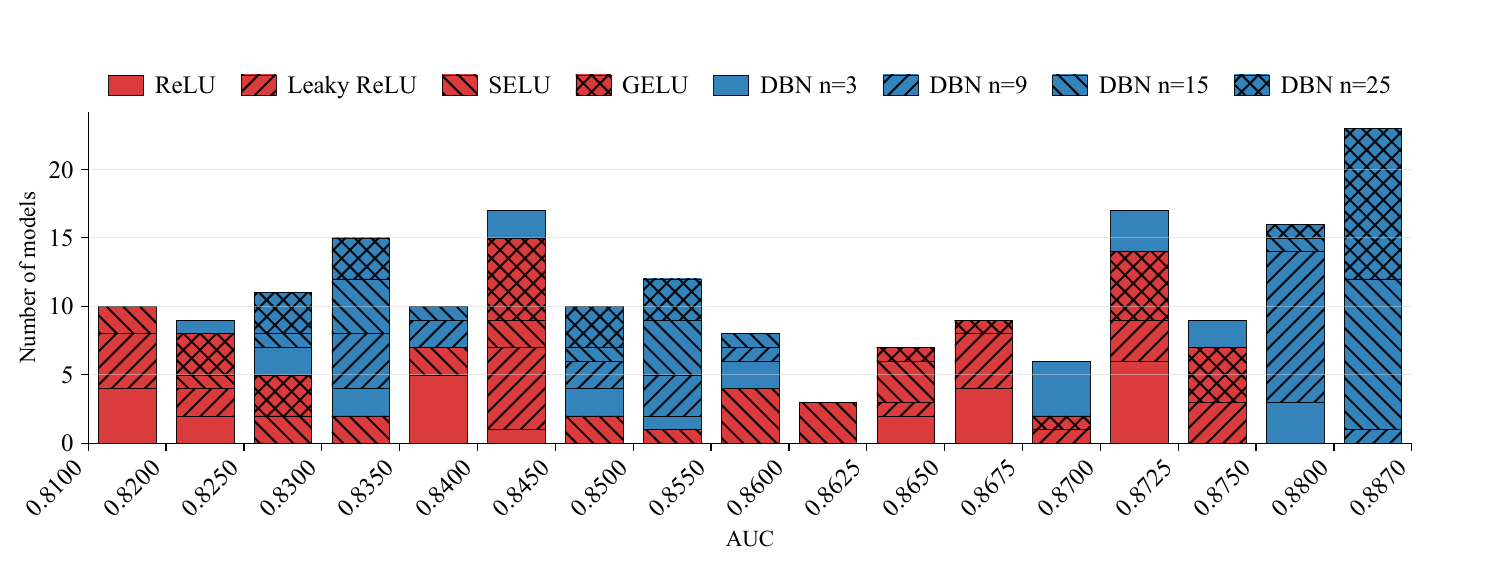}
        \caption{HIGGS, no residual}
        \label{fig:higgs_nores_auc_hist_app}
    \end{subfigure}
    \hfill
    \begin{subfigure}[b]{0.48\textwidth}
        \centering
        \includegraphics[width=\textwidth]{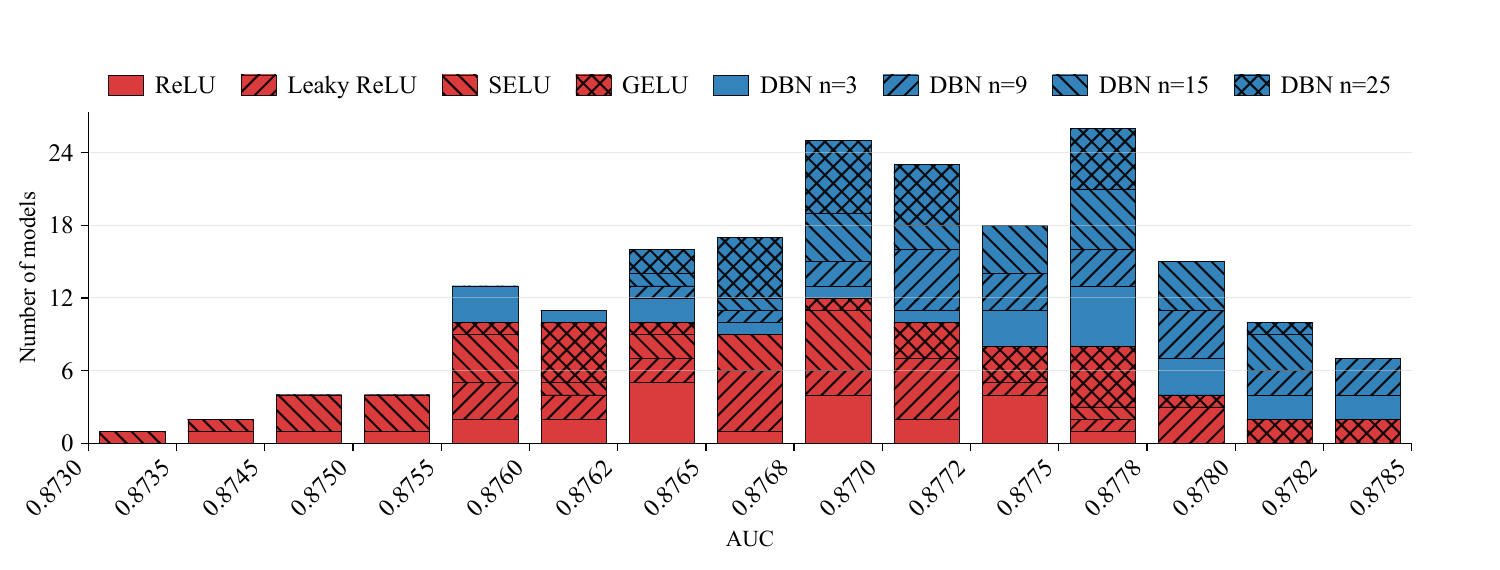}
        \caption{SUSY, no residual}
        \label{fig:susy_nores_auc_hist_app}
    \end{subfigure}
    \caption{\textbf{Distribution of validation AUC (no residual connections).} Same format as Figure~\ref{fig:auc_loss_histograms}. DBNs occupy the high-AUC tail in both datasets. On HIGGS~(a), all buckets above $0.875$ are DBN-exclusive.}
    \label{fig:auc_histograms_nores}
\end{figure}

\begin{figure}[t]
    \centering
    \begin{subfigure}[b]{0.49\textwidth}
        \centering
        \includegraphics[width=\textwidth]{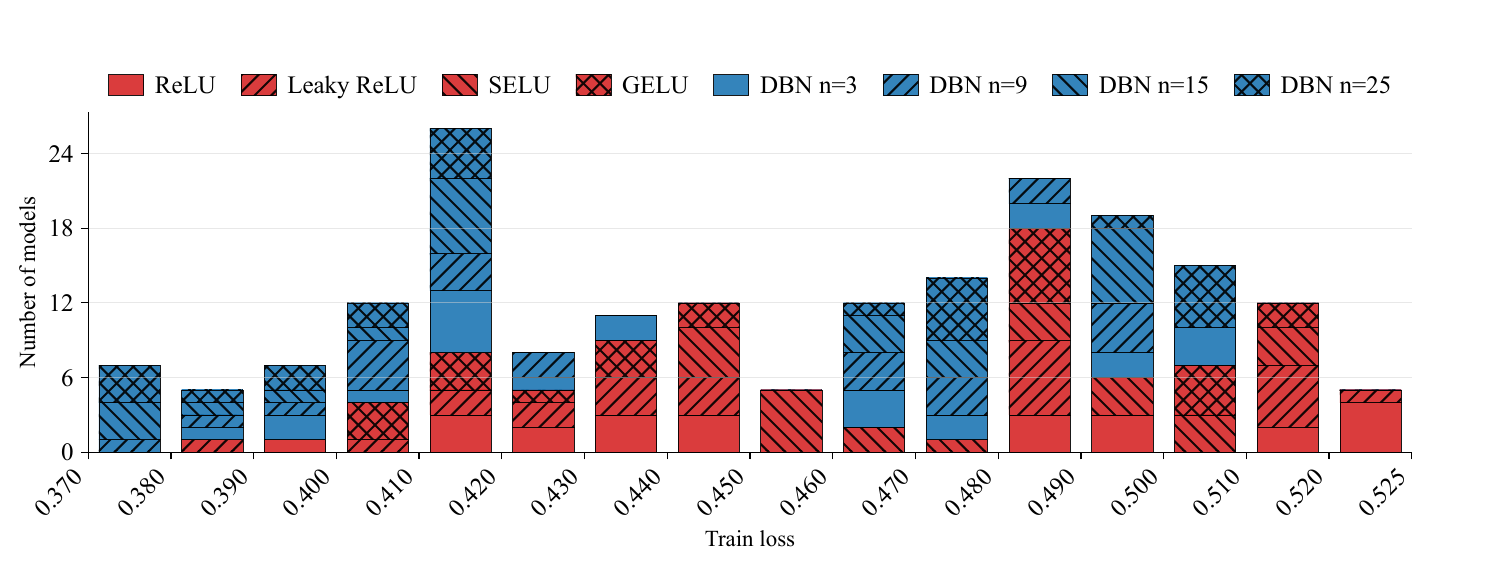}
        \caption{HIGGS, no residual}
        \label{fig:higgs_nores_loss_hist_app}
    \end{subfigure}
    \hfill
    \begin{subfigure}[b]{0.49\textwidth}
        \centering
        \includegraphics[width=\textwidth]{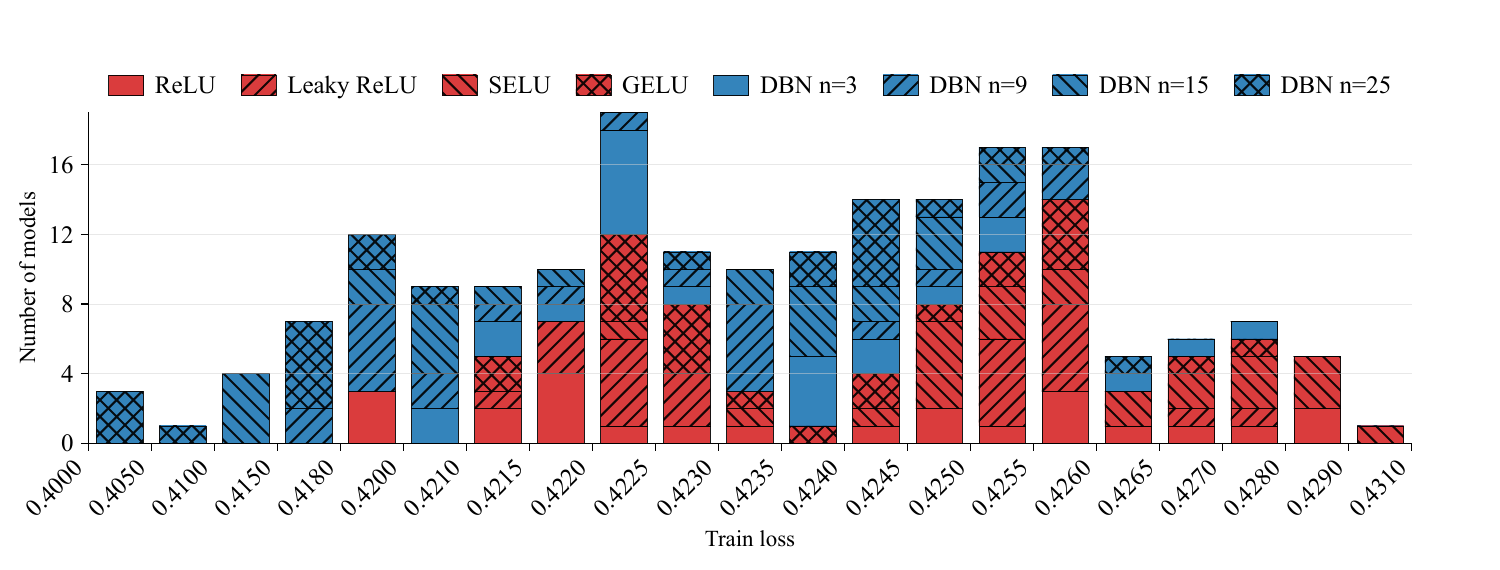}
        \caption{SUSY, no residual}
        \label{fig:susy_nores_loss_hist_app}
    \end{subfigure}
    \caption{\textbf{Distribution of training loss (no residual connections).} Same format as Figure~\ref{fig:auc_loss_histograms}. On HIGGS~(a), the entire loss region below $0.40$ is DBN-exclusive. On SUSY~(b), losses below $0.415$ are dominated by higher-degree DBNs.}
    \label{fig:loss_histograms_nores}
\end{figure}


\section{Wall-Clock Timing and Memory Analysis}
\label{app:timing}

A natural concern with polynomial activations is computational overhead: evaluating a degree-$n$ Bernstein polynomial and its gradient is more expensive per forward/backward pass than a scalar ReLU. To mitigate this, we implement the closed-form Bernstein derivative directly in a custom Triton kernel, bypassing PyTorch's autograd graph for the activation backward pass. This yields substantial speedups and memory savings compared to a na\"ive autograd implementation, since the analytic gradient avoids materializing the full computational graph of the polynomial evaluation.

Tables~\ref{tab:runtime_higgs_res}--\ref{tab:runtime_higgs_nores} report the mean total training runtime (in seconds, averaged over 3 seeds) for all relu and DBN architecture combinations. Several patterns emerge.

\paragraph{Per-epoch overhead is moderate.}
At the smallest architectures ($5{\times}16$), DBN training takes roughly $1.5{\times}$--$2{\times}$ longer than ReLU in total wall-clock time. This reflects the per-epoch overhead of polynomial evaluation, which dominates when the network is small and the number of epochs is comparable. As architectures grow, this ratio narrows: the matrix multiplications in the linear layers increasingly dominate the per-iteration cost, and the polynomial evaluation becomes a smaller fraction of the total.

\paragraph{Faster convergence offsets per-epoch cost.}
Since DBNs converge in fewer epochs (Table~\ref{tab:crossover_higgs_res}), the total wall-clock time is often comparable to or even \emph{lower} than ReLU despite the higher per-epoch cost. On HIGGS with residual connections (Table~\ref{tab:runtime_higgs_res}), several large architectures show DBN training times below ReLU:
\begin{itemize}[nosep,leftmargin=1.5em]
\item $100{\times}128$: ReLU takes $2{,}858$s; DBN $n{=}9$ takes $2{,}427$s ($\mathbf{15\%}$ \textbf{faster}).
\item $100{\times}256$: ReLU takes $5{,}213$s; DBN $n{=}9$ takes $4{,}104$s ($\mathbf{21\%}$ \textbf{faster}).
\item $150{\times}128$: ReLU takes $4{,}057$s; DBN $n{=}9$ takes $3{,}770$s ($7\%$ faster).
\item $150{\times}256$: ReLU takes $7{,}723$s; DBN $n{=}9$ takes $6{,}359$s ($\mathbf{18\%}$ \textbf{faster}).
\end{itemize}
This occurs because these DBN models train for roughly half the epochs of their ReLU counterparts (e.g., $150{\times}256$: DBN $n{=}9$ runs 24 epochs vs.\ ReLU's 50), and the per-epoch overhead factor (${\sim}1.7{\times}$) is more than compensated by the ${\sim}2{\times}$ epoch reduction. Notably, $n{=}9$ and $n{=}3$ tend to have the lowest runtimes among DBN variants since their polynomial evaluation is cheapest.

\paragraph{Memory overhead.}
The Bernstein coefficients and the domain-clamping bookkeeping introduce additional memory. For the largest architecture ($150{\times}256$), peak GPU memory usage is $14.1$\,GB compared to $9.2$\,GB for ReLU---a $1.5{\times}$ overhead. This remains within the capacity of a single modern GPU (e.g., NVIDIA A30 with 24\,GB). The Triton kernel further reduces memory pressure by computing the activation gradient in-place without storing intermediate polynomial terms.

\paragraph{Summary.}
While DBNs have a higher per-epoch cost than ReLU (roughly $1.5$--$2{\times}$), their faster convergence frequently results in comparable or lower \emph{total} training times, particularly at the deeper, wider architectures where the parameter efficiency gains from \S\ref{sec:exp_compression} are largest. The overhead is further mitigated by our Triton-based closed-form gradient implementation.


\begin{table}[t]
\centering
\setlength{\tabcolsep}{5pt}
\caption{\textbf{Total training runtime on HIGGS (residual connections).} Mean wall-clock time in seconds, averaged over 3 seeds. At deeper architectures ($L \geq 100$), several DBN configurations train faster than ReLU in total time despite higher per-epoch cost, because they converge in fewer epochs (cf.\ Table~\ref{tab:crossover_higgs_res}).}
\label{tab:runtime_higgs_res}
\begin{tabular}{lrrrrr}
\toprule
\textbf{Arch} & \textbf{ReLU} & \multicolumn{4}{c}{\textbf{DBN}} \\
\cmidrule(lr){3-6}
($L{\times}W$) & (s) & $n=3$ & $n=9$ & $n=15$ & $n=25$ \\
\midrule
5$\times$16 & 233 & 445 & 405 & 339 & 327 \\
5$\times$32 & 248 & 412 & 463 & 391 & 403 \\
5$\times$128 & 300 & 608 & 552 & 838 & 594 \\
5$\times$256 & 440 & 531 & 391 & 488 & 428 \\
\midrule
15$\times$16 & 567 & 718 & 620 & 826 & 935 \\
15$\times$32 & 558 & 862 & 901 & 735 & 873 \\
15$\times$128 & 491 & 877 & 623 & 932 & 757 \\
15$\times$256 & 636 & 780 & 670 & 1,041 & 782 \\
\midrule
50$\times$16 & 1,764 & 2,636 & 2,916 & 2,792 & 2,921 \\
50$\times$32 & 1,354 & 2,557 & 2,418 & 2,429 & 2,275 \\
50$\times$128 & 1,320 & 1,490 & 1,405 & 1,861 & 1,450 \\
50$\times$256 & 1,981 & 2,143 & 2,093 & 3,041 & 2,480 \\
\midrule
100$\times$16 & 3,178 & 3,924 & 4,368 & 6,603 & 5,677 \\
100$\times$32 & 3,517 & 4,128 & 4,062 & 3,782 & 3,627 \\
100$\times$128 & 2,858 & 2,477 & 2,427 & 3,450 & 2,900 \\
100$\times$256 & 5,213 & 4,607 & 4,104 & 5,970 & 4,880 \\
\midrule
150$\times$16 & 4,725 & 8,090 & 6,276 & 6,069 & 7,564 \\
150$\times$32 & 5,221 & 7,270 & 5,859 & 5,697 & 5,042 \\
150$\times$128 & 4,057 & 4,488 & 3,770 & 5,380 & 4,458 \\
150$\times$256 & 7,723 & 6,616 & 6,359 & 9,358 & 7,587 \\
\bottomrule
\end{tabular}
\end{table}

\begin{table}[t]
\centering
\setlength{\tabcolsep}{5pt}
\caption{\textbf{Total training runtime on SUSY (residual connections).} Mean wall-clock time in seconds, averaged over 3 seeds. The DBN overhead ratio is generally $1.5$--$2{\times}$ at small architectures, narrowing at larger scales as linear-layer computation dominates.}
\label{tab:runtime_susy_res}
\begin{tabular}{lrrrrr}
\toprule
\textbf{Arch} & \textbf{ReLU} & \multicolumn{4}{c}{\textbf{DBN}} \\
\cmidrule(lr){3-6}
($L{\times}W$) & (s) & $n=3$ & $n=9$ & $n=15$ & $n=25$ \\
\midrule
5$\times$16 & 59 & 119 & 114 & 106 & 106 \\
5$\times$32 & 60 & 115 & 108 & 101 & 203 \\
5$\times$128 & 58 & 203 & 175 & 102 & 93 \\
5$\times$256 & 82 & 135 & 139 & 166 & 125 \\
\midrule
15$\times$16 & 113 & 207 & 204 & 223 & 208 \\
15$\times$32 & 137 & 254 & 257 & 168 & 206 \\
15$\times$128 & 124 & 216 & 186 & 253 & 213 \\
15$\times$256 & 236 & 407 & 277 & 464 & 366 \\
\midrule
50$\times$16 & 451 & 678 & 844 & 813 & 680 \\
50$\times$32 & 374 & 636 & 706 & 615 & 547 \\
50$\times$128 & 472 & 564 & 611 & 927 & 694 \\
50$\times$256 & 819 & 1,114 & 1,009 & 1,498 & 1,081 \\
\midrule
100$\times$16 & 950 & 1,110 & 1,113 & 1,213 & 1,371 \\
100$\times$32 & 858 & 1,392 & 1,229 & 1,161 & 964 \\
100$\times$128 & 886 & 1,067 & 1,216 & 1,567 & 1,209 \\
100$\times$256 & 1,787 & 2,474 & 1,975 & 2,884 & 2,591 \\
\midrule
150$\times$16 & 1,666 & 1,977 & 2,162 & 2,024 & 1,804 \\
150$\times$32 & 1,190 & 1,781 & 1,798 & 1,418 & 1,377 \\
150$\times$128 & 1,620 & 1,937 & 1,976 & 2,195 & 1,802 \\
150$\times$256 & 2,335 & 3,718 & 2,893 & 4,206 & 3,289 \\
\bottomrule
\end{tabular}
\end{table}

\begin{table}[t]
\centering
\setlength{\tabcolsep}{5pt}
\caption{\textbf{Total training runtime on SUSY (no residual connections).} Mean wall-clock time in seconds, averaged over 3 seeds. At $5{\times}256$, several DBN configurations ($n{=}3$, $n{=}9$) are faster than ReLU.}
\label{tab:runtime_susy_nores}
\begin{tabular}{lrrrrr}
\toprule
\textbf{Arch} & \textbf{ReLU} & \multicolumn{4}{c}{\textbf{DBN}} \\
\cmidrule(lr){3-6}
($L{\times}W$) & (s) & $n=3$ & $n=9$ & $n=15$ & $n=25$ \\
\midrule
5$\times$16 & 54 & 71 & 78 & 80 & 79 \\
5$\times$32 & 63 & 85 & 78 & 79 & 76 \\
5$\times$128 & 60 & 85 & 83 & 96 & 90 \\
5$\times$256 & 151 & 132 & 115 & 316 & 212 \\
\midrule
15$\times$16 & 262 & 337 & 495 & 482 & 299 \\
15$\times$32 & 207 & 215 & 253 & 179 & 164 \\
15$\times$128 & 173 & 180 & 178 & 236 & 216 \\
15$\times$256 & 209 & 313 & 264 & 464 & 457 \\
\bottomrule
\end{tabular}
\end{table}

\begin{table}[t]
\centering
\setlength{\tabcolsep}{5pt}
\caption{\textbf{Total training runtime on HIGGS (no residual connections).} Mean wall-clock time in seconds, averaged over 3 seeds. The overhead ratio is roughly $1.5$--$2{\times}$ at $L{=}5$, consistent with the residual setting at the same depth.}
\label{tab:runtime_higgs_nores}
\begin{tabular}{lrrrrr}
\toprule
\textbf{Arch} & \textbf{ReLU} & \multicolumn{4}{c}{\textbf{DBN}} \\
\cmidrule(lr){3-6}
($L{\times}W$) & (s) & $n=3$ & $n=9$ & $n=15$ & $n=25$ \\
\midrule
5$\times$16 & 203 & 343 & 347 & 305 & 312 \\
5$\times$32 & 226 & 448 & 555 & 522 & 530 \\
5$\times$128 & 269 & 814 & 581 & 718 & 567 \\
5$\times$256 & 340 & 494 & 414 & 626 & 484 \\
\midrule
15$\times$16 & 475 & 853 & 889 & 779 & 1,072 \\
15$\times$32 & 456 & 921 & 925 & 850 & 841 \\
15$\times$128 & 645 & 1,061 & 924 & 1,202 & 1,116 \\
15$\times$256 & 736 & 1,076 & 800 & 1,166 & 930 \\
\bottomrule
\end{tabular}
\end{table}




\end{document}